\documentclass[10pt,twocolumn,letterpaper]{article}

\usepackage{cvpr}              

\usepackage{graphicx}
\usepackage{amsmath}
\usepackage{amssymb}
\usepackage{amsthm}
\usepackage{booktabs}
\usepackage{soul}
\usepackage{dsfont}
\usepackage{algorithm}
\usepackage{algorithmic}
\usepackage{multirow}
\usepackage{enumitem}
\usepackage{physics}
\usepackage[accsupp]{axessibility} 

\newcounter{alphasect}
\def\alphainsection{0}

\let\oldsection=\section
\def\section{%
  \ifnum\alphainsection=1%
    \addtocounter{alphasect}{1}
  \fi%
\oldsection}%

\renewcommand\thesection{%
  \ifnum\alphainsection=1%
    \Alph{alphasect}%
  \else%
    \arabic{section}%
  \fi%
}%

\newenvironment{alphasection}{%
  \ifnum\alphainsection=1%
    \errhelp={Let other blocks end at the beginning of the next block.}
    \errmessage{Nested Alpha section not allowed}
  \fi%
  \setcounter{alphasect}{0}
  \def\alphainsection{1}
}{%
  \setcounter{alphasect}{0}
  \def\alphainsection{0}
}%

\makeatletter
\@namedef{ver@everyshi.sty}{}
\makeatother
\usepackage{tikz}

\usepackage{xr}

\makeatletter
\newcommand*{\addFileDependency}[1]{
  \typeout{(#1)}
  \@addtofilelist{#1}
  \IfFileExists{#1}{}{\typeout{No file #1.}}
}
\makeatother

\newcommand*{\myexternaldocument}[1]{%
    \externaldocument{#1}%
    \addFileDependency{#1.tex}%
    \addFileDependency{#1.aux}%
}
\myexternaldocument{supplementary}

\usepackage[breaklinks,colorlinks]{hyperref}

\usepackage[capitalize]{cleveref}

\newcommand{\bone}{\mathbf{1}}
\newcommand{\cQ}{\mathcal{Q}}

\newcommand{\bA}{\mathbf{A}}
\newcommand{\cA}{\mathcal{A}}
\newcommand{\ba}{\mathbf{a}}

\newcommand{\cB}{\mathcal{B}}

\newcommand{\bb}{\mathbf{b}}

\newcommand{\cC}{\mathcal{C}}

\newcommand{\cD}{\mathcal{D}}
\newcommand{\cP}{P}
\newcommand{\bH}{\mathbf{H}}

\newcommand{\cS}{\mathcal{S}}
\newcommand{\bI}{\mathbf{I}}
\newcommand{\cI}{\mathcal{I}}

\newcommand{\bJ}{\mathbf{J}}

\newcommand{\cO}{\mathcal{O}}

\newcommand{\bp}{\mathbf{p}}
\newcommand{\bQ}{\mathbf{Q}}
\newcommand{\bq}{\mathbf{q}}

\newcommand{\cR}{\mathcal{R}}
\newcommand{\bbR}{\mathbb{R}}
\newcommand{\bS}{\mathbf{S}}

\newcommand{\bt}{\mathbf{t}}

\newcommand{\bv}{\mathbf{v}}

\newcommand{\bx}{\mathbf{x}}

\newcommand{\bz}{\mathbf{z}}

\newcommand{\cV}{\mathcal{V}}

\newcommand{\pnt}{\lambda}
\newcommand{\dclim}{\bar{\lambda}}
\newcommand{\dcrate}{\gamma}
\newcommand{\dcstep}{\bar{M}}

\newcommand{\comment}[1]{}

\newcommand\tikzmark[1]{\tikz[remember picture,overlay]\coordinate (#1);}

\newcommand{\tikzcircle}[2][red,fill=red]{\tikz[baseline=-0.7ex]\draw[#1,radius=#2] (0,0) circle ;}%
\newcommand{\tikzcircleedge}[2][red,fill=white]{\tikz[baseline=-0.7ex]\draw[#1,radius=#2, very thick] (0,0) circle ;}%
\newcommand{\tikzrect}[2][red,fill=red]{\tikz[baseline=0.7ex]\draw[#1,very thick] (0,0) rectangle (0.3,0.3);}%

\newtheorem{definition}{Definition}
\newtheorem{property}{Property}

\newtheorem{claim}{Claim}

\definecolor{turbo}{RGB}{249, 202, 36}
\definecolor{carmine_pink}{RGB}{235, 77, 75}
\definecolor{pure_apple}{RGB}{106, 176, 76}
\definecolor{steel_pink}{RGB}{190, 46, 221}
\definecolor{soaring_eagle}{RGB}{149, 175, 192}
\definecolor{blurple}{RGB}{75, 52, 212}
\definecolor{heliotrope}{RGB}{224, 86, 253}
\definecolor{cyan}{RGB}{0, 255, 255}
\definecolor{lime}{RGB}{0, 255, 0}
\definecolor{middle_blue}{RGB}{126, 214, 223}
\definecolor{midnight_blue}{RGB}{44,62,80}
\definecolor{june_buo}{RGB}{186, 220, 88}
\definecolor{june_buo}{RGB}{186, 220, 88}

\crefname{section}{Sec.}{Secs.}
\Crefname{section}{Section}{Sections}
\Crefname{table}{Table}{Tables}
\crefname{table}{Tab.}{Tabs.}

\def\cvprPaperID{785} 
\def\confName{CVPR}
\def\confYear{2022}

\begin{document}

\title{A Hybrid Quantum-Classical Algorithm for Robust Fitting}
\author{Anh-Dzung Doan~$^{1}$ \hspace{1em} Michele Sasdelli~$^{1}$ \hspace{1em} David Suter~$^{2}$ \hspace{1em} Tat-Jun Chin~$^{1}$\\
$^{1}$ School of Computer Science, The University of Adelaide, South Australia \\
$^{2}$ Centre for AI\&ML, School of Science, Edith Cowan University, Western Australia}

\maketitle

\begin{abstract}
Fitting geometric models onto outlier contaminated data is provably intractable. Many computer vision systems rely on random sampling heuristics to solve robust fitting, which do not provide optimality guarantees and error bounds. It is therefore critical to develop novel approaches that can bridge the gap between exact solutions that are costly, and fast heuristics that offer no quality assurances. In this paper, we propose a hybrid quantum-classical algorithm for robust fitting. Our core contribution is a novel robust fitting formulation that solves a sequence of integer programs and terminates with a global solution or an error bound. The combinatorial subproblems are amenable to a quantum annealer, which helps to tighten the bound efficiently. While our usage of quantum computing does not surmount the fundamental intractability of robust fitting, by providing error bounds our algorithm is a practical improvement over randomised heuristics. Moreover, our work represents a concrete application of quantum computing in computer vision. We present results obtained using an actual quantum computer (D-Wave Advantage) and via simulation\footnote{Source code: https://github.com/dadung/HQC-robust-fitting}.
\end{abstract}

\section{Introduction}\label{sec:intro}

Imperfections in sensing and processing in computer vision inevitably generate data that contain outliers. Therefore, it is necessary for vision pipelines to be robust against outliers in order to mitigate their harmful effects.

In 3D vision, where a major goal is to recover the scene structure and camera motion, a basic task is to fit a geometric model onto noisy and outlier prone measurements. This is often achieved through the consensus maximisation framework~\cite{tj_book}: given $N$ data points $\cD = \{ \bp_i\}_{i=1}^N$ and a target geometric model parametrised by $\bx \in \mathbb{R}^d$, let $\cP_N$ be the power set of index set $\{1,\dots,N \}$. We aim to solve
\begin{align}\label{eq:maxcon}
\begin{aligned}
     \max_{\cI \in \cP_N,~\bx \in \mathbb{R}^d} \quad & |\cI|\\
    \textrm{s.t.} \quad & r_i(\bx) \le \epsilon \;\;\; \forall i \in \cI,
\end{aligned}
\end{align}
where $r_i(\bx)$ is the residual of point $\bp_i$ w.r.t.~$\bx$, and $\epsilon$ is a given inlier threshold. The form of $r_i(\bx)$ depends on the specific geometric model (more details in Sec.~\ref{sec:reformulating}). A candidate solution $(\cI,\bx)$ consists of a consensus set $\cI$ and its ``witness'' (an estimate) $\bx$, where the points in $\cI$ are the inliers of $\bx$. Problem~\eqref{eq:maxcon} seeks the maximum consensus set $\cI^\ast$, whose witness $\bx^\ast$ is a robust estimate of the model\footnote{This also depends on using a correct $\epsilon$. The large volume of works that apply consensus maximisation suggest setting $\epsilon$ is usually not a concern.}.

Many computer vision systems employ random sampling heuristics, i.e., RANSAC~\cite{ransac} and its variants (e.g.,~\cite{tran2014sampling, pham2011simultaneous,loransac,graphcutrs,magsacplusplus, mlesac}), for consensus maximisation. The basic idea is to repeatedly fit the model on randomly sampled minimal subsets of $\cD$, and return the $\tilde{\bx}$ with the largest consensus set $\tilde{\cI}$. Such heuristics can only approximate~\eqref{eq:maxcon} and generally do not provide optimality guarantees or error characterisation, e.g., a tight bound on the discrepancy $|\cI^\ast| - |\tilde{\cI}|$. Moreover, $\tilde{\bx}$ is subject to randomness, and postprocessing or reruns are often executed to vet the result.

Unfortunately, consensus maximisation is provably intractable~\cite{tj18nphard,antonante2020outlier}, hence there is little hope in finding efficient algorithms that can solve it exactly. While there has been active research into globally optimal algorithms~\cite{chin2015efficient,cai2019consensus,li2009consensus,parra2015guaranteed,olsson2008polynomial}, such techniques are realistic only for small input instances (small $d$, $N$ and/or number of outliers~\cite{chin2015efficient}).

Bridging the gap between exact algorithms that are costly and randomised heuristics that offer no quality assurances is an important research direction in robust fitting with practical ramifications. Towards this aim, deterministic approximate algorithms~\cite{huu2017ep,zhipeng2018ibco, purkait2018maximum, le2018non,wen2019efficient} eschew exhaustive search (e.g., branch-and-bound) and randomisation, and instead adopt deterministic subroutines such as convex optimisation, proximal splitting, etc. These methods avoid the vagaries of random sampling, and some can even guarantee convergence~\cite{huu2017ep,purkait2018maximum,le2018non}. However, none of them provide error bounds. Indeed, complexity results~\cite{tj18nphard,antonante2020outlier} also preclude efficient approximate solutions with error bounds.

Partly buoyed by the dominance of deep learning in computer vision, learning-based solutions to robust geometric fitting have been developed~\cite{brachmann2017dsac,ranftl2018deep,truong2021unsupervised}. Such techniques leverage statistics in large datasets to learn a mapping from the input instance to the desired solution. Despite showing promising results in benchmark datasets, learning methods do not provide optimality guarantees and error bounds. Whether the learned model can generalise is also a concern.


To summarise, existing algorithms for robust fitting, particularly those targeted at consensus maximisation, have yet to satisfactorily solve the problem. It is thus worthwhile to investigate novel approaches based on new insights.

\vspace{-1em}
\paragraph{Our contributions}

We propose a new approach that leverages quantum computing for consensus maximisation. Our core contribution is a consensus maximisation algorithm that iteratively solves a sequence of integer programs and terminates with either $\bx^\ast$ or a suboptimal solution $\tilde{\bx}$ with a known error bound $|\cI^\ast| - |\tilde{\cI}| \le \rho$. The integer programs are amenable to a quantum annealer~\cite[Chap.~8]{scherer2019mathematics}, which is utilised to tighten the bound efficiently. Since our method employs convex subroutines and random sampling, it is a \emph{hybrid quantum-classical} algorithm~\cite{gyongyosi2019survey, perdomo2018opportunities, bravyi2020hybrid, kim2021advancing}.

We will present results using an actual quantum computer, the D-Wave Advantage~\cite{dwaveadvantage1.1}, as well as simulation. While our technique does not yet outperform state-of-the-art algorithms, in part due to the limitations of current quantum technology, our work represents a concrete application of quantum computing in computer vision. We hope to inspire future efforts on this topic in the community.


\vspace{-0.5em}
\section{Related work}\label{sec:related_works}

In Sec.~\ref{sec:intro}, we have provided an overview of robust fitting and recent algorithmic advances. We thus focus our survey on quantum computing in computer vision.

Many quantum methods have been proposed for image processing~\cite{dendukuri2018image,venegas2003storing,caraiman2012image,yan2016survey}, image recognition~\cite{neven2012qboost,nguyen2018image,dendukuri2019defining}, and object detection~\cite{li2020quantum}. Also, several methods explored the tasks of classification and training a deep neural network~\cite{sasdelli2021quantum,khoshaman2018quantum, nguyen2021deep}. Recently, Golyanik and Theobalt~\cite{golyanik2020quantum} proposed a practical quantum algorithm for rotation estimation to align two point sets. Their basic idea is to discretise rotation matrices to formulate the problem to the quadratic unconstrained binary optimization (QUBO), which can be solved by quantum annealers. Benkner et al.~\cite{benkner2020adiabatic} proposed to solve the graph matching problem through formulating the quadratic assignment problem (QAP) to the QUBO using the penalty approach. They conducted the experiment and provided analysis on the quantum computer D-Wave 2000Q. However, the limitation of quantum computers precluded them from solving large problems. To address this issue, instead of enforcing a penalty to QAP, Q-Match~\cite{seelbach2021q} was proposed to iteratively select and solve subproblems of QAP, which allows D-Wave annealers to efficiently deal with large problems. Another interesting work is QuantumSync~\cite{birdal2021quantum}, which addresses the synchronisation problem in the context of multi-image matching. This work carefully formulated the synchronisation problem to the QUBO, which was then validated on D-Wave Advantage. 


The closest work to ours is~\cite{TJ_ACCV_2020}, who proposed a quantum solution for robust fitting. However, there are non-trivial differences: first,~\cite{TJ_ACCV_2020} estimates \emph{per-point influences} (a measure of outlyingness)~\cite{tennakoon2021consensus,suter2020monotone} for outlier removal instead of consensus maximisation. Second, their algorithm is based on the gate computing model, which is fundamentally different from the quantum annealing approach adopted in our work. Third, the results in~\cite{TJ_ACCV_2020} are only based on simulation; we will compare against~\cite{TJ_ACCV_2020} on this basis in Sec.~\ref{sec:experiments}.

\section{Reformulating consensus maximisation}\label{sec:reformulating}

In this section, we describe our novel reformulation for consensus maximisation and relevant theoretical results, before presenting the usage of quantum annealing in Sec.~\ref{sec:qubo_for_scp} and the overall algorithm in Sec.~\ref{sec:main_algo}.

\subsection{Preliminaries}\label{sec:prelim}

Following~\cite{TJ_ACCV_2020}, we consider residuals $r_i(\bx)$ that are quasiconvex, which encapsulates many geometric models of interest in computer vision~\cite{kahl08}. Formally, if the set
\begin{align}
\{ \bx \in \mathbb{R}^d \mid r_i(\bx) \le \alpha \}
\end{align}
is convex for all $\alpha \ge 0$, then $r_i(\bx)$ is quasiconvex. Note that assuming quasiconvex residuals does not reduce the computational hardness of consensus maximisation~\cite{tj18nphard}.

For $\cC \in \cP_N$, define the minimax problem
\begin{equation}\label{eq:minimax}
    g(\cC) = \min_{\bx \in \mathbb{R}^d} \; \max_{i \in \cC} \;  r_i(\bx).
\end{equation}
For quasiconvex $r_i(\bx)$,~\eqref{eq:minimax} is a quasiconvex program~\cite{eppstein,amenta}, which is polynomial-time solvable. Note that $g(\cC) \le \epsilon$ implies that $\cC$ is a consensus set, since all the points in $\cC$ are within error $\epsilon$ to the extremiser of~\eqref{eq:minimax}.

Define the ``feasibility test''
\begin{align}\label{eq:feasibility}
f(\cC) = \begin{cases} 0 & \textrm{if}~g(\cC) \le \epsilon; \\ 1 & \textrm{otherwise}. \end{cases}
\end{align}
Any $\cC$ such that $f(\cC) = 0$ implies that $\cC$ is a consensus set.

\noindent Problem~\eqref{eq:maxcon} can thus be restated as
\begin{align}\label{eq:maxcon2}
\begin{aligned}
     \max_{\cI \in \cP_N} \quad & |\cI|, \quad & 
    \textrm{s.t.} \quad & f(\cI) = 0,
\end{aligned}
\end{align}
with the witness $\bx$ for any feasible $\cI$ obtainable through computing $g(\cI)$ to evaluate $f(\cI)$.

Given a consensus set $\cI$ with witness $\bx$, the points in the complement $\cO = \{1,\dots,N\}\setminus \cI$ are the outliers to $\bx$. The ``dual'' of problem~\eqref{eq:maxcon2} is therefore
\begin{align}\label{eq:maxcon3}
\begin{aligned}
     \min_{\cO \in \cP_N} \quad & |\cO|, \quad &
    \textrm{s.t.} \quad & f(\{1,\dots,N\}\setminus \cO) = 0,
\end{aligned}
\end{align}
i.e., find the model with the least number of outliers.

\begin{definition}[True inliers and true outliers]
Let $\cI^\ast$ be the maximum consensus set and $\cO^\ast = \{1,\dots,N\}\setminus \cI^\ast$. We call $\cI^\ast$ the ``true inliers'' and $\cO^\ast$ the ``true outliers''.
\end{definition}

\begin{property}[Monotonicity]\label{prop:monotonicity}
For~\eqref{eq:minimax} with quasiconvex residuals, given subsets $\mathcal{P}, \cQ, \cR \in \cP_N$ with $\mathcal{P} \subseteq \cQ \subseteq \cR$, we have $g(\mathcal{P}) \le g(\cQ) \le g(\cR)$.
By extension, we also have that $f(\mathcal{P}) \le f(\cQ) \le f(\cR)$.
See~\cite{eppstein,amenta} for more details.
\end{property}

Intuitively, adding points to a feasible subset can only potentially make it infeasible; the converse cannot be true. This leads to the following crucial concept.


\begin{definition}[Basis]\label{def:basis}
A basis $\cB \subset \{1,\dots,N\}$ is a subset such that $g(\cB') < g(\cB)$ for every $\cB' \subset \cB$.
\end{definition}

Intuitively, removing any point from a basis $\cB$ will cause the minimax value of the subset to shrink.

\begin{property}[Combinatorial dimension]\label{prop:combdim}
The combinatorial dimension $\delta$ of minimax problem~\eqref{eq:minimax} is the upper bound on the size of bases~\cite{eppstein,amenta}. For quasiconvex $r_i(\bx)$, $\delta = 2d+1$.
\end{property}

\begin{claim}\label{claim:true_outlier_in_basis}
If basis $\cB$ is infeasible, i.e., $f(\cB) = 1$, then $\left|\cB \cap \cO^\ast \right| \ge 1$, i.e., an infeasible basis $\cB$ contains at least one true outlier.
\end{claim}
\begin{proof}
See Sec.~\ref{supp:sec:proof_true_outlier_in_basis} in supplementary material.
\end{proof}

\subsection{Hypergraph vertex cover} \label{sec:vertex_cover}

Define the binary $N$-vector
\begin{align}
\bz = \left[ z_1, \dots, z_N \right] \in \{ 0,1\}^N,
\end{align}
where the set of indices corresponding to nonzero $z_i$'s are
\begin{align}
\cC_\bz = \{ i \in \{1,\dots,N\} \mid z_i = 1 \}.
\label{eq:O_from_z}
\end{align}


\noindent The outlier minimisation problem~\eqref{eq:maxcon3} can be reexpressed as
\begin{align}\label{eq:maxcon4}
\begin{aligned}
     \min_{\bz \in \{0,1\}^N} \quad & \| \bz \|_1, \quad &
    \textrm{s.t.} \;\; & f(\{1,\dots,N\}\setminus\cC_\bz) = 0,
\end{aligned}
\end{align}
where $z_i = 1$ implies the $i$-th point is removed as an outlier.

Let $\{ \bb_{(k)} \}^{K}_{k=1}$ be $K$ binary $N$-vectors that correspond to \emph{all infeasible bases} of the problem, i.e., for each $k$,
\begin{align}
    f(\cC_{\bb_{(k)}}) = 1, \;\;\;\; \|\bb_{(k)}\|_1 \le \delta,
\end{align}
where the latter appeals to the combinatorial dimension (Property~\ref{prop:combdim}). Also, the number of infeasible bases $K = O(N^\delta)$. Define the hypergraph $H$ with vertex set $\cV$ and hyperedge set $E$ respectively as
\begin{align}\label{eq:hypergraph}
    H = \{ \cV, E \}, \;\; \cV = \{1,\dots,N\}, \;\; E = \{ \cC_{\bb_{(k)}} \}^{K}_{k=1}.
\end{align}
Recall that hypergraphs are a generalisation of graphs, where a hyperedge can be incident with more than two vertices~\cite{agarwal2005beyond}. In our hypergraph~\eqref{eq:hypergraph}, each hyperedge connects vertices that form an infeasible basis; see Fig.~\ref{fig:dataspace_hypergraph}.

\begin{claim}\label{mainclaim}
A subset $\cI \subseteq \cV$ is a consensus set \emph{iff} it is an independent set of hypergraph $H$.
\end{claim}
\begin{proof}
See Sec.~\ref{supp:sec:proof_main_claim} in supplementary material.
\end{proof}

Claim~\ref{mainclaim} proves that finding the maximum consensus set $\cI^\ast$ is equivalent to finding the maximum independent set of $H$. Since the complement of an independent set is a vertex cover, it justifies to minimise the vertex cover
\begin{align}\label{eq:scp_ilp}\tag{VC}
\begin{aligned}
     \min_{\bz \in \{0,1\}^N} \quad & \| \bz \|_1 \\
    \textrm{s.t.} \quad & \bb_{(k)}^T \bz\ge 1, \;\;\; \forall k = 1,\dots,K,
\end{aligned}
\end{align}
which is an 0-1 integer linear program (ILP). Setting $z_i = 1$ implies removing the $i$-th vertex, and the constraints ensure that all hyperedges are ``covered'', i.e., at least one vertex in each hyperedge is removed (cf.~Claim~\ref{claim:true_outlier_in_basis}).

The hypergraph formalism has been applied previously in geometric fitting~\cite{agarwal2005beyond,purkait2016clustering,liu2012efficient,lin2019hypergraph}. However, the target problem in~\cite{agarwal2005beyond,purkait2016clustering,liu2012efficient,lin2019hypergraph} was higher order clustering (e.g., via hypergraph cuts), which is very distinct from our aims.

Formulation~\eqref{eq:scp_ilp} is impractical for two reasons:
\begin{itemize}[leftmargin=1em,itemsep=2pt,parsep=0pt,topsep=2pt]
    \item Hypergraph vertex cover is intractable;
    \item The number of hyperedges in $H$ is exponential.
\end{itemize}
However, the form~\eqref{eq:scp_ilp} is amenable to a quantum annealer, as we will show in Sec.~\ref{sec:qubo_for_scp}. To deal with the number of hyperedges, we propose a hybrid quantum-classical algorithm in Sec.~\ref{sec:main_algo} that incrementally generates hyperedges.

\begin{figure*}[th]
    \centering
    \mbox{
        \subfloat[  ][ {Data points with index set $\{1, \dots, 12 \}$. Infeasible bases $\tikzcircleedge[carmine_pink, fill=white]{4pt} = \{1,5,6\}$, $\tikzcircleedge[middle_blue]{4pt} = \{ 2,7,8\}$, $\tikzcircleedge[pure_apple, fill=white]{4pt} = \{3,9,10 \}$, $\tikzcircleedge[heliotrope]{4pt} = \{4,11,12 \}$, $\tikzcircleedge[turbo, fill=white]{4pt} = \{1,2,3 \}$, $\tikzcircleedge[midnight_blue, fill=white]{4pt} = \{2,3,4 \}$. Outlier set $\cO = \{1,2,3,4\}$. Consensus set $\cI = \{5, 6, 7, 8, 9, 10, 11, 12 \}$. } ]{
            \includegraphics[width=0.3\textwidth]{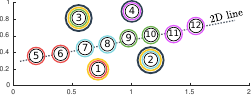}.
            \label{fig:dataspace}
        }
        
        \hspace{2em}
        \subfloat[][ {Hypergraph with vertex set $\cV = \{1, \dots, 12 \}$. Hyperedges $\tikzrect[carmine_pink]{4pt} = \{1,5,6\}$, $\tikzrect[middle_blue]{4pt} = \{ 2,7,8\}$, $\tikzrect[pure_apple]{4pt} = \{3,9,10 \}$, $\tikzrect[heliotrope]{4pt} = \{4,11,12 \}$, $\tikzrect[turbo]{4pt} = \{1,2,3 \}$, $\tikzrect[midnight_blue]{4pt} = \{2,3,4 \}$. Vertex cover $ = \{1,2,3,4\}$. Independent set $=\{5, 6, 7, 8, 9, 10, 11, 12 \}$ }  ]{
            \includegraphics[width=0.5\textwidth]{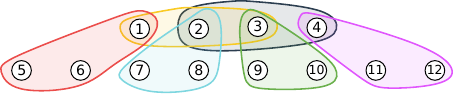}
            \label{fig:hypergraph}
        }
    }
    \vspace{-1em}
    \caption{(a) Suppose we fit a line $\bx \in \bbR^2$ on $N = 12$ points $\{(a_i, b_i) \}^N_{i=1}$ on the plane, with residual $r_i(\bx) = \begin{Vmatrix}b_i -\bx[a_i,1] \end{Vmatrix}_2$. The maximum consensus size is $8$. Six infeasible bases are also plotted (colour coded, note that there are in total 281 infeasible bases). (b) Equivalent  hypergraph $H$ for (a) based on our construction (Sec.~\ref{sec:vertex_cover}). The infeasible basis $\equiv$ hyperedge. Claim~\ref{mainclaim} implies the maximum consensus set $\equiv$ maximum independent set, and the minimum outlier set $\equiv$ minimum vertex cover.}
    \label{fig:dataspace_hypergraph}
\end{figure*}

\section{Quantum solution}\label{sec:qubo_for_scp}

We first provide a basic introduction to quantum annealing, before describing our quantum treatment of~\eqref{eq:scp_ilp}.

\subsection{Quantum annealing}\label{sec:background_quantum_annealing}

A quantum annealer solves optimisation problems through energy minimisation of a physical system. A Hamiltonian defines the energy profile of a quantum system, which is composed of a number of interacting qubits. The system's state is initialised at the lowest energy of the initial Hamiltonian and annealed  such that the its final state gives the desired solution.
At the end of the annealing,
the Hamiltonian can be obtained from
the following model
\begin{align}\label{eq:bin_quad}
    \sum_{n} Q_{nn}q_n  + \sum_{n < m} Q_{nm} q_n q_m =  \bq^T \bQ \bq.
\end{align}
The measurement collapses the 
$N$-qubit quantum state
into $\bq = \left[q_1, q_2, \dots, q_N\right]$, where 
$q_n \in \{0,1\}$, $\bQ \in \mathbb{R}^{N \times N}$.
The elements of $\bQ$ define the couplings between qubits  and their biases; see~\cite[Chap.~8]{scherer2019mathematics} for more details.

A quantum annealer solves a problem of the form
\begin{align}\label{eq:generic_qubo}
    \min_{\bq \in \{0,1\}^N} \bq^T\bQ\bq,
\end{align}
which is the quadratic unconstrained binary optimisation (QUBO). QUBO is intractable on a classical machine,
but a quantum annealer, by virtue of the physical processes described above, \emph{may} solve the problem efficiently. It allows $N$-qubits to evolve
through superposed and entangled states
(quantum tunnelling), and $\bq$ is obtained from the final measurement; see Sec.~\ref{sec:practical} on practical limitations.

\subsection{Hypergraph vertex cover as QUBO}

To simplify description of the main algorithm in Sec.~\ref{sec:main_algo}, we first generalise~\eqref{eq:scp_ilp}. Let $A$ be a subset of the hyperedges $E$ of the hypergraph $H$:
\begin{align}
    A = \{ \cC_{\ba_{(m)}} \}^{M}_{m=1} \subseteq E = \{ \cC_{\bb_{(k)}} \}^{K}_{k=1}.
\end{align}
Define the 0-1 ILP
\begin{align}\label{eq:scp_ilp_matrix}
    \begin{aligned}
        I(A) =  \min_{\bz \in \{0,1\}^N} \quad & \| \bz \|_1, \quad &
        \textrm{s.t.} \quad & \bA^T\bz \geq \bone_M,
    \end{aligned}
\end{align}
where $\bA \in \{0,1\}^{N \times M}$ is obtained by horizontally stacking the binary $N$-vectors $\{ \ba_{(m)} \}^{M}_{m=1}$ corresponding to the hyperedges in $A$, and $\bone_M$ is the vector of $M$ ones.


We can recover~\eqref{eq:scp_ilp} from~\eqref{eq:scp_ilp_matrix} by setting $A = E$. Moroever, since $A \subseteq E$, it is clear that $I(A) \le I(E).$

To formulate~\eqref{eq:scp_ilp_matrix} as a QUBO, we first convert the inequalities into equalities. Define $\delta^\prime = \delta - 1$. For each constraint $\ba_{(m)}^T\bz \ge 1$ in~\eqref{eq:scp_ilp_matrix}, we incorporate $\delta^\prime$ binary slack variables $\bt_{(m)} = \left[ \begin{matrix} t_{m,1} & \dots & t_{m,\delta^\prime} \end{matrix} \right]^T$ into the constraint
\begin{align}
    \ba_{(m)}^T\bz - \bt_{(m)}^T\bone_{\delta^\prime} = 1;
\end{align}
recall that each $\ba_{(m)}$ has exactly $\delta$ elements with value $1$. All $M$ equality constraints can be expressed in matrix form

\begin{align}\label{eq:Hconstraints}
    \bH_A \left[ \begin{matrix}\bv^T & 1 \end{matrix} \right]^T = 0, 
\end{align}
where
$
\bv = \left[ \begin{matrix} \bz^T & \bt_{(1)}^T & \dots & \bt_{(m)}^T & \dots & \bt_{(M)}^T \end{matrix} \right]^T \in \{0,1\}^{N+\delta^\prime M}$, 
$\bH_A = \begin{bmatrix} \bA^T && -\bS && -\bone_M\end{bmatrix} \in \{0,1\}^{M\times(N+\delta^\prime M+1)}$, 
$\bS = \bI_M \otimes \bone_{\delta^\prime}^T$, the $M \times M$ identity matrix $\bI_M$, and Kronecker product $\otimes$. Also, the objective $\| \bz \|_1$ can be expressed in the quadratic form
\begin{align}
\left[ \begin{matrix}\bv^T & 1 \end{matrix}\right]\bJ \left[ \begin{matrix}\bv^T & 1 \end{matrix} \right]^T,
\end{align}
with
$
\bJ =\begin{bmatrix} \bI_N && \mathbf{0}_{N \times (\delta^\prime M+1)} \\ \mathbf{0}_{(\delta^\prime M+1) \times N} && \mathbf{0}_{(\delta^\prime M+1) \times (\delta^\prime M+1)} \end{bmatrix},
$
where $\mathbf{0}$ is a zero matrix with the size specified in its subscript, and $\bI_N$ is the $N\times N$ identity matrix. 



With penalty parameter $\pnt>0$, we lift the constraints~\eqref{eq:Hconstraints} into the objective to yield the QUBO
\begin{align} 
    \begin{aligned}
        Q_\pnt(A) = \min_{\bv \in \{0,1\}^{N+\delta^\prime M}}  & \left[ \begin{matrix}\bv^T  \, 1 \end{matrix}\right] (\bJ + \pnt\bH_A^T\bH_A) \left[ \begin{matrix}\bv^T \, 1 \end{matrix}\right]^T.
    \end{aligned}
    \label{eq:scp_qubo}
\end{align}
Further algebraic manipulation is required to remove the constant $1$ from~\eqref{eq:scp_qubo} before exactly matching~\eqref{eq:generic_qubo}; see Sec.~\ref{supp:sec:reformulate_qubo_to_dwave} in the supp.~material for details. In the following, we will discuss solving~\eqref{eq:scp_qubo} using a quantum annealer.



\subsection{Practical considerations and limitations}\label{sec:practical}

We frame the discussion here in the context of SOTA quantum annealer---the D-Wave Advantage~\cite{dwaveadvantage1.1}.

\vspace{-1em}
\paragraph{Challenges}

Problem~\eqref{eq:scp_qubo} is an application of the quadratic penalty method~\cite[Chap.~17]{nocedalandwright}. While fundamental results exist that allow $Q_\lambda(A)$ to equal $I(A)$, they invariably require $\lambda$ to approach a large value. However, the precision of D-Wave Advantage is limited to 4-5 bits~\cite{ice,dorband2018extending}, which precludes the usage of large penalty parameters.

Second, although there are $>$5000 qubits in D-Wave Advantage, the topology of quantum processing unit (QPU) rules out a fully connected model, i.e., the $\bQ$ matrix allowable is not dense~\cite{dattani2019pegasus, neven2009nips}. Given an arbitrary $\bQ$, a \emph{minor embedding} step~\cite{robertson1995graph, cai2014practical,boothby2016fast} is required to map the QUBO onto the QPU topology. The embedding consumes extra physical qubits reducing the number of physical qubits available.

As alluded in Sec.~\ref{sec:background_quantum_annealing}, the annealing process ``gradually'' transitions (NB: by human scale the transition is rapid) the quantum system from the initial Hamiltonian to the final Hamiltonian. Current quantum annealers are not able to completely isolate external noise from the process, which affects the quality of the solution.

To obtain a useful solution, during the annealing process, the system must have a non-negligible probability of staying in the lowest energy state. If the system jumps to a higher energy state, it will fail in solving the QUBO~\eqref{eq:scp_qubo} optimally. The spectral gap is the minimum gap between the lowest and the second lowest (higher) energy states, which affects the probability of staying in the lowest energy state; see~\cite{farhi2000quantum,whatisqadwave} for details. We will investigate the spectral gap issue in Sec.~\ref{supp:sec:spectral_gap} of the supplementary material.



\vspace{-1em}
\paragraph{Why quantum annealing?}

The above issues limit the scale of problems and quality of solutions attainable with current quantum annealers. However, quantum technology is advancing steadily, and the vision community should be prepared for potential breakthroughs, as like-minded colleagues are also advocating~\cite{golyanik2020quantum, benkner2020adiabatic, seelbach2021q,birdal2021quantum, TJ_ACCV_2020, sasdelli2021quantum}. Moreover, our main algorithm combines quantum and classical computation to leverage the strengths of both paradigms.

\section{Main algorithm}\label{sec:main_algo}
\vspace{-0.5em} 

Alg.~\ref{main_algo} presents our overall algorithm. At its core, our algorithm aims to solve~\eqref{eq:scp_ilp}, i.e., find the minimum outlier set, but by incrementally generating the hyperedges $A \subseteq E$. Other main characteristics of the algorithm are:
\begin{itemize}[leftmargin=1em,itemsep=2pt,parsep=0pt,topsep=2pt]
\item At each iteration, the QUBO~\eqref{eq:scp_qubo} based on the current hyperedges $A$ is solved using quantum annealing.
\item The penalty $\lambda$ for~\eqref{eq:scp_qubo} decays following a schedule defined by hyperparameters $\dclim$, $\dcrate$ and $\dcstep$ (Step~\ref{algo:penalty_decay}).
\item Hyperedges are sampled from a candidate vertex set $\cV^\prime$, which is updated based on the current results (Sec.~\ref{sec:heuristic_basis}).
\end{itemize}
The algorithm terminates with the best estimate $\bz_\text{best}$ of the minimum outlier set and the sampled hyperedges $A$.

In the following, we show how the outputs of Alg.~\ref{main_algo} can be used to derive an error bound for consensus maximisation, and the rationale of our hyperedge sampling technique.


\begin{algorithm}[h]\centering
	\begin{algorithmic}[1]
	\REQUIRE Data $\cD = \{ \bp_i\}_{i=1}^N$, inlier threshold $\epsilon$, maximum iterations $M$, penalty $\pnt$ with decay parameters $\dclim$, $\dcrate$, $\dcstep$.
	\vspace{-1em}
	\STATE Initialise hyperedge set $A \leftarrow \emptyset$, candidate vertices $\cV^\prime \leftarrow \cV$, best outlier set $\bz_\text{best} \leftarrow \bone_{N}$.
	\FOR{$m$ = 1 to $M$}
	    \STATE  $\ba_{(m)} \leftarrow$ Active set of $\cV^\prime$ (see Sec.~\ref{sec:heuristic_basis}). \label{algo:select_basis}
	    \STATE $A \leftarrow A \cup \{ \cC_{\ba_{(m)}} \}$.
	    \IF{$m$  mod $\dcstep$ = 0}
	    \STATE $\pnt \leftarrow \text{max}(\pnt.\dcrate, \dclim)$.\label{algo:penalty_decay}
	    \ENDIF
	    \STATE $\bv  = [\begin{matrix}\bz & \bt\end{matrix}] \leftarrow$ Solve~\eqref{eq:scp_qubo} using quantum annealing \label{algo:solve_qubo}
	    \IF{$f(\cV\setminus \cC_\bz) = 0$}
            \STATE $\cI \leftarrow \cV\setminus \cC_\bz$ (found a consensus set). \label{algo:foundconsensus}
	        \IF {$\|\bz\|_1 < \|\bz_\text{best}\|_1$}
	            \STATE $\bz_\text{best} \leftarrow \bz$.\label{algo:store_best}
	        \ENDIF
	        \STATE $\cV^\prime \leftarrow \cC_{\bz} \cup \{$random subset of $\cI\}$ (see Sec.~\ref{sec:heuristic_basis}). \label{algo:Xprime_eq_O_with_I0}
	    \ELSE
	        \STATE $\cV^\prime \leftarrow \cV\setminus \cC_\bz$ (see Sec.~\ref{sec:heuristic_basis}). \label{algo:Dprime_eq_I}
	    \ENDIF
	\ENDFOR
	\RETURN Outlier set estimate $\bz_\text{best}$ and hyperedge set $A$.
	\end{algorithmic}
	\caption{Hybrid Quantum-Classical Robust Fitting. Note: Only Step~\ref{algo:solve_qubo} invokes the quantum annealer.}
	\label{main_algo}
\end{algorithm}

\subsection{Error bound}
\label{sec:error_bound}
\vspace{-0.5em} 

Consider the relaxation of~\eqref{eq:scp_ilp_matrix}
\begin{align}\label{eq:scp_lp_matrix}
    \begin{aligned}
        LP(A) =  \min_{\bz \in [0,1]^N} \quad &  \| \bz \|_1, \quad &
        \textrm{s.t.} \quad & \bA^T\bz \geq \bone_M
    \end{aligned}
\end{align}
which is a linear program. We must have that
\begin{align}
    LP(A) \le I(A) \le I(E).
\end{align}
Due to the factors in Sec.~\ref{sec:practical}, the solution $\bz$ by the quantum annealer on~\eqref{eq:scp_qubo} can be suboptimal. Given the best solution $\bz_\text{best}$, if the $\cI_{\text{best}} = \cV \setminus \cC_{\bz_\text{best}}$ is a consensus set, by Claim~\ref{mainclaim}, $\bz_\text{best}$ is a vertex cover of~\eqref{eq:scp_ilp}. We must have that
\begin{align}
    LP(A) \le I(E) \le \| \bz_\text{best} \|_1.
\end{align}
Using the fact that $|\cI^\ast| = N - I(E)$, we thus have
\begin{align}
    |\cI^\ast| - |\cI_{\text{best}}| \le \| \bz_\text{best} \|_1 - LP(A).
\end{align}
If the RHS is $0$, then $\cI_\text{best}$ is the globally optimal solution.



\subsection{Heuristic for sampling hyperedges}
\label{sec:heuristic_basis}
\vspace{-0.5em} 

Recall that a hyperedge is an infeasible basis. A simple way to generate hyperedges is to randomly sample $\delta$-subsets from $\cV$ until we find an infeasible subset, which will not be efficient. To improve efficiency, our sampling technique maintains a candidate vertex set $\cV^\prime \subseteq \cV$ where $f(\cV^\prime) = 1$, and takes the active set $\cS$ of $\cV^\prime$~\cite{eppstein,amenta}, where
\begin{align}
    g(\cS) = g(\cV^\prime),
\end{align}
as a hyperedge. Intuitively, the active set of $\cV^\prime$ is a basis with equal value with $\cV^\prime$. To generate diverse hyperedges, two strategies are employed to maintain $\cV^\prime$:
\begin{itemize}[leftmargin=1em,itemsep=2pt,parsep=0pt,topsep=2pt]
\item Take $\cV^\prime = \cV \setminus \cC_{\bz}$ if it is not a consensus set (Step~\ref{algo:Dprime_eq_I});
\item If a new consensus set $\cI = \cV \setminus \cC_\bz$ is found (Step~\ref{algo:foundconsensus}), set $\cV^\prime$ as the union of $\cC_\bz$ with a random subset of $\cI$ (Step~\ref{algo:Xprime_eq_O_with_I0}).
\end{itemize}

\subsection{Hyperparameter selection}
\vspace{-0.5em} 

The penalty value and decay schedule play important roles in Alg.~\ref{main_algo} to quickly find a consensus set and tighten the error bound, see Sec.~\ref{supp:sec:penalty_decay} in supplementary material for details. The precise values for the parameters used and/or investigated in our experiments will be provided in Sec.~\ref{sec:experiments}.

\section{Experiments}\label{sec:experiments}

\subsection{Synthetic data}\label{sec:synthetic}

We first examine the performance of D-Wave Advantage (version 1.1)~\cite{dwaveadvantage1.1} on our robust fitting formulation via synthetic data. We generated 2D points $\cD = \{(a_i,b_i)\}^{N}_{i=1}$ for 1D linear regression ($x \in \bbR^1$) with residual $r_i(x) = |a_i x - b_i|$, where $0 \le a_i, b_i \le 1$. For a randomly chosen ground truth $x$, a proportion of the points are corrupted with Gaussian noise of $\sigma_{in} = 0.1$ to form inliers, with the rest by Gaussian noise of $\sigma_{out} = 1.5$ to simulate outliers.

Due to the cost of accessing the QPU, our results in this subsection were not derived from many data repetitions. However, each QPU input instance was invoked with 10,000 anneals, which typically consumed $\approx$ 1.3 seconds. Also, see Sec.~\ref{supp:sec:embedding} in the supp.~material for details on embedding~\eqref{eq:scp_qubo} (e.g., number of qubits, runtime) onto the QPU.

\vspace{-1em}
\paragraph{Comparison between CPU and QPU.} We first compare CPU and QPU performance on our QUBO~\eqref{eq:scp_qubo} (i.e., independent of Alg.~\ref{main_algo}), with $A$ containing all hyperedges $E$. The CPU solver used was Gurobi~\cite{gurobi}, which solves~\eqref{eq:scp_qubo} exactly via exhaustive search, hence practical only for small instances. Fig.~\ref{fig:synthetic_cpu_qpu} plots the number of outliers $\|\bz\|_1$ (lower is better) optimised by the solvers as a function of
\begin{itemize}[leftmargin=1em,itemsep=2pt,parsep=0pt,topsep=2pt]
    \item Penalty $\pnt \in [0.1,100]$, with $N = 50$, outlier ratio $=0.2$.
    \item Outlier ratio $\in [0.1,0.6]$, with $N = 20$, $\pnt = 1.0$.
    \item $N \in [10,100]$, with $\pnt = 1.0$, outlier ratio $=0.2$.
\end{itemize}
As expected (see Sec.~\ref{sec:practical}), the gap in quality between the QPU solution and the ``ground truth'' provided by Gurobi increased with the examined parameters, indicating that the QPU is more reliable on ``easier'' instances of~\eqref{eq:scp_qubo}.


\begin{figure}[ht]\centering
	\centering
	
	    \subfloat[][Effect of penalty $\pnt$ ($N = 50$, outlier ratio $= 0.2$)]
		{
			\includegraphics[width=0.4\textwidth]{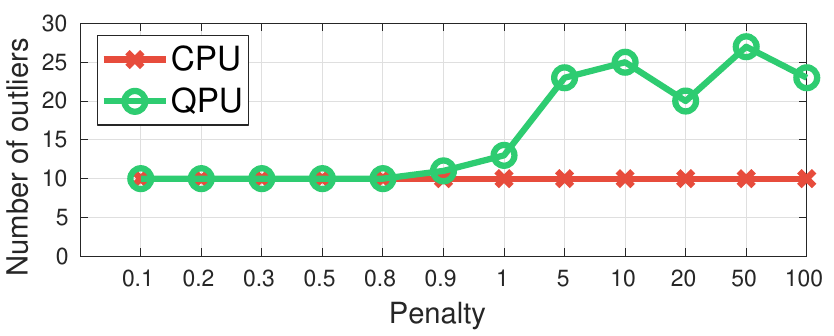}
			\label{fig:synthetic_cpu_qpu:penalty}
		}
		
		\subfloat[][Effect of outlier ratio ($N = 20$, $\pnt = 1.0$)]
		{
			\includegraphics[width=0.4\textwidth]{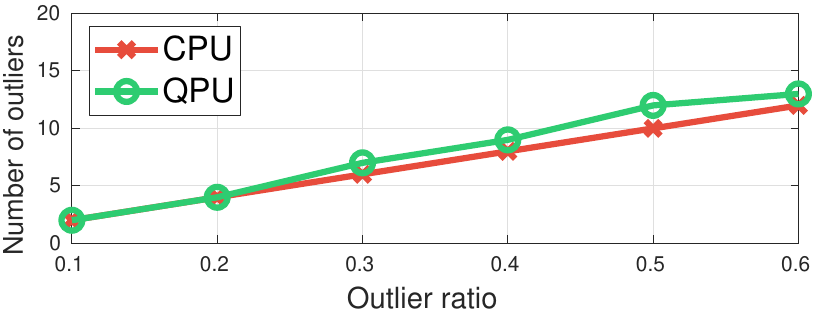}
			\label{fig:synthetic_cpu_qpu:outlier_ratio}
		}
		
		\subfloat[][Effect of data size $N$ ($\pnt = 1.0$, outlier ratio $=0.2$)]
		{
			\includegraphics[width=0.4\textwidth]{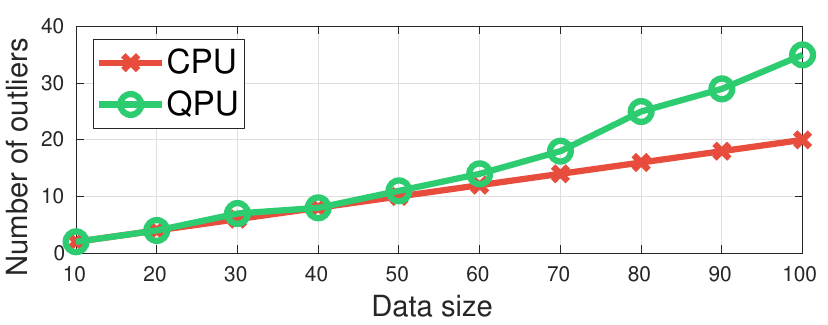}
			\label{fig:synthetic_cpu_qpu:data_size}
		}
    \vspace{-0.5em}	
	\caption{Comparisons between CPU and QPU on QUBO~\eqref{eq:scp_qubo}.}
	\label{fig:synthetic_cpu_qpu}
	
\end{figure}

\vspace{-1em}
\paragraph{Main algorithm}

Fig.~\ref{fig:synthetic_error_bound} illustrates running Alg.~\ref{main_algo} on synthetic 1D linear regression instances with $N=20$, $50$, and $100$ points, each with outlier ratio $0.2$. The QUBO subroutine~\eqref{eq:scp_qubo} in the main algorithm was solved using the QPU with $\pnt = 1.0$ (no $\lambda$ decay was done). The values $\| \bz \|_1$ and $LP(A)$ were plotted as a function of the size of $A$, i.e., number of hyperedges. The results mainly illustrate the feasibility of solving robust fitting using quantum annealing.


\begin{figure}[h]
	\centering
	\includegraphics[width=0.45\textwidth]{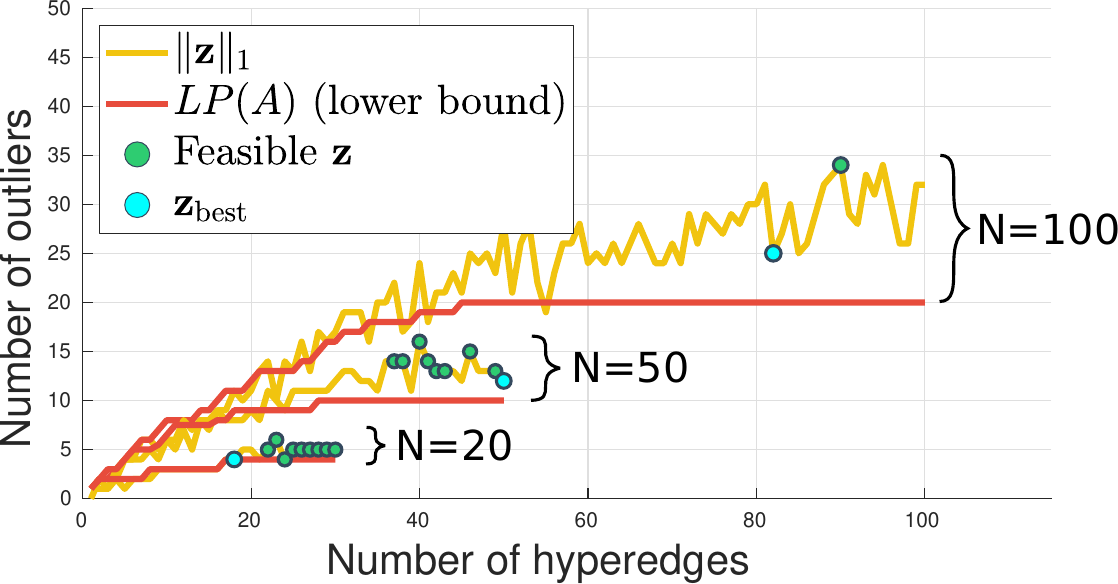}
	\caption{Number of outliers $\| \bz \|_1$ optimised by QPU and lower bound $LP(A)$, plotted across the iterations of Alg.~\ref{main_algo}.}
	\vspace{-0.5em}
	\label{fig:synthetic_error_bound}
\end{figure}

\vspace{-1em}
\paragraph{Comparing to simulated annealing} In the context of Alg.~\ref{main_algo}, we compared quantum annealing (QA) and simulated annealing (SA)~\cite{simanneal} (on CPU with 10,000 anneals) in solving the QUBO subroutine (Line~\ref{algo:solve_qubo} of Alg.~\ref{main_algo}). A synthetic 1D linear regression instance with $N=20$ and outlier ratio $0.2$ was generated. The penalty $\pnt$ was set to $0.5$ (no $\lambda$ decay was done). Fig.~\ref{fig:qa-sa} shows the runtime of QA and SA across the iterations of Alg.~\ref{main_algo} (for QA, the cost of embedding~\eqref{eq:scp_qubo} onto the QPU was excluded; again, see Sec.~\ref{supp:sec:embedding} in the supp.~material for details), and the Hamming distance between the $\bz$'s found by the methods in each iteration.

\begin{figure}[h]
    \centering
    \includegraphics[width=0.42\textwidth]{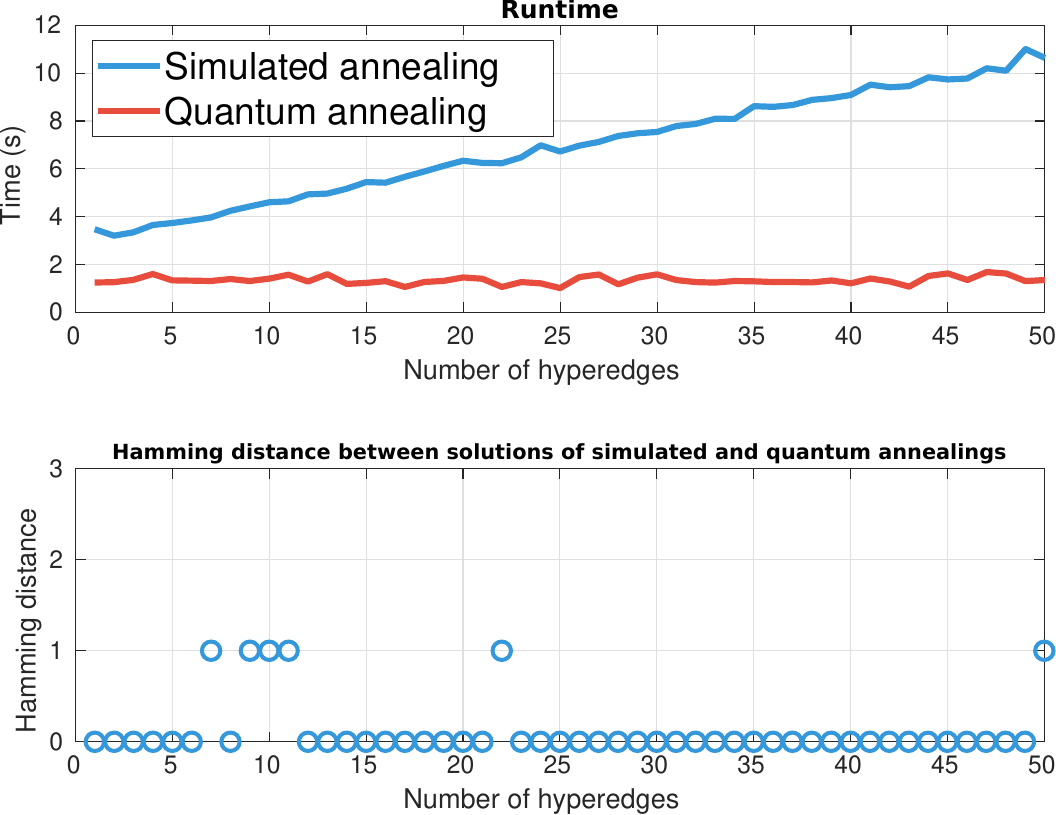}
    \vspace{-0.5em}	
    \caption{Comparing quantum annealing (on D-Wave Advantage) and simulated annealing (on classical computer).}
    \vspace{-0.5em}	
    \label{fig:qa-sa}
\end{figure}

The results illustrate that the runtime of SA (on CPU) grew steadily as the number of sampled hyperedges $A$ increased, whereas the runtime of QA remained largely constant across the iterations, which suggests that the underlying physical processes of QA were not affected significantly by problem size (as long as the problem ``fits'' on the QPU). Also, Fig.~\ref{fig:qa-sa} shows the solutions obtained by QA and SA are largely the same; this supports using SA in place of QA to examine the efficacy of Alg.~\ref{main_algo} on larger sized real data.


\subsection{Real data}\label{sec:realdata}

We tested our method on real data for fundamental matrix estimation and multi-view triangulation. We used SA (on CPU) in place of QA to allow Alg.~\ref{main_algo} to handle bigger problems. Two variants of Alg.~\ref{main_algo} were executed:
\begin{itemize}[leftmargin=1em,itemsep=2pt,parsep=0pt,topsep=2pt]
    \item Alg.~\ref{main_algo}-E, where the algorithm was terminated as soon as a consensus set was found (Line \ref{algo:foundconsensus}).
    \item Alg.~\ref{main_algo}-F, where the algorithm was run until the maximum iterations $M$ ($300$ for fund.~matrix, $200$ for triangulation).
\end{itemize}

We compared our method to i) random sampling methods: RANSAC (RS)~\cite{ransac}, LO-RANSAC (LRS)~\cite{loransac}, and Fixing LO-RANSAC (FLRS)~\cite{fixingloransac}, ii) deterministic algorithms: Exact penalty (EP)~\cite{huu2017ep}, and Iterative biconvex optimization (IBCO)~\cite{zhipeng2018ibco}, and iii) quantum robust fitting (QRF)~\cite{TJ_ACCV_2020}. Each method was run 100 times and average results were reported. All experiments were conducted on a system with 2.6~GHz processor and 16~GB of RAM.

\vspace{-0.5em}
\subsubsection{Fundamental matrix estimation.} \label{sec:experiment_fund}
\vspace{-0.5em} 

\begin{table*}[h]
    	\centering
		\includegraphics[width=1\textwidth]{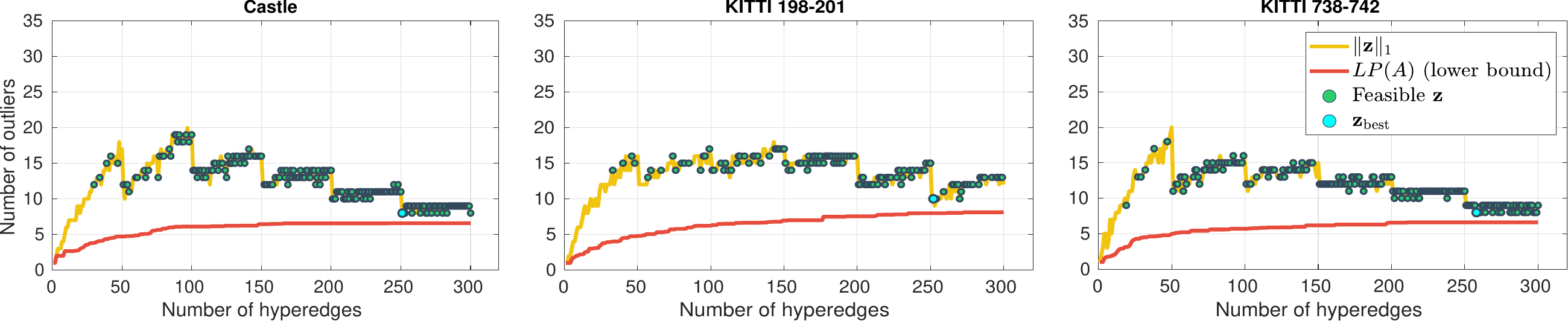}
		\vspace{-2em}
    	\captionof{figure}{Fundamental matrix estimation, where number of outliers $\| \bz \|_1$ and lower bound $LP(A)$, plotted across the iterations of Alg.~\ref{main_algo}-F.}
    	\label{fig:fund_bound}
    \vspace*{0.5em}  
    
    \footnotesize
    \begin{tabular}{|c|c|c|c|c|c|c|c|c|c|}
        \hline
         \multicolumn{2}{|c|}{Method} & RS~\cite{ransac} & LRS~\cite{loransac} & FLRS~\cite{fixingloransac} & EP~\cite{huu2017ep} & IBCO~\cite{zhipeng2018ibco} & QRF~\cite{TJ_ACCV_2020} & Alg.~\ref{main_algo}-E & Alg.~\ref{main_algo}-F \\
         \hline \hline
         Castle &   $|\cI|$ (Error bound) & 74 ($-$) & 74 ($-$) & 74 ($-$) & 70 ($-$) & 76 ($-$) & 73 ($-$) & 72 (8.17) & 76 (1.41)\\
         $N = 84$ & Time (s) & 0.20 & 0.11 & 0.20 & 0.25 & 0.34 & 199.48 & 18.07 & 1998.87 \\
         \hline
         
         Valbonne & $|\cI|$ (Error bound) & 34 ($-$) & 36 ($-$) & 36 ($-$) & 33 ($-$) & 38 ($-$) & 29 ($-$) & 36 (6.00) & 36 (4.00)\\
         $N = 45$ & Time (s) & 0.21 & 0.20 & 0.31 & 0.34 & 0.44 & 110.30 & 6.71 & 1915.82 \\
         \hline
         
         Zoom & $|\cI|$ (Error bound) & 90 ($-$) & 91 ($-$) & 91 ($-$) & 92 ($-$) & 95 ($-$) & 89 ($-$) & 93 (9.91) & 94 (3.64) \\
         $N = 108$ & Time (s) & 0.31 & 0.29 & 0.14 & 0.21 & 0.35 & 257.03 & 92.35 & 2109.13\\
         \hline
         
         KITTI 104-108 & $|\cI|$ (Error bound) & 309 ($-$) & 313 ($-$) & 312 ($-$) & 318 ($-$) & 321 ($-$) & 256 ($-$) & 320 (9.91) & 324 (2.30) \\
         $N = 337$ & Time (s) & 0.04 & 0.04 & 0.07 & 0.28 & 0.39 & 799.33 & 137.26 & 2408.04 \\
         \hline
         
         KITTI 198-201 & $|\cI|$ (Error bound) & 306 ($-$) & 308 ($-$) & 307 ($-$) & 308 ($-$) & 312 ($-$) & 309 ($-$) & 308 (10.00) & 312 (1.89)\\
         $N=322$ & Time (s) & 0.05 & 0.13 & 0.07 & 0.23 & 0.42 & 774.06 & 36.15 & 2350.39 \\
         \hline 
         
         KITTI 738-742 & $|\cI|$ (Error bound) & 481 ($-$) & 483 ($-$) & 483 ($-$) & 491 ($-$) & 492 ($-$) & 447 ($-$) & 492 (5.88) & 493 (1.39)\\
         $N =501$ & Time (s) & 0.05 & 0.18 & 0.23 & 0.53 & 0.61 & 1160.12 & 22.46 & 2506.04\\
         \hline
    \end{tabular}
    \vspace{-1em}
    \captionof{table}{Fundamental matrix estimation results. Alg.~\ref{main_algo} employed simulated annealing (quantum annealing could be faster by $10$ times). Only Alg.~\ref{main_algo}  amongst all methods compared here returned error bounds.}
    \label{tab:fund_compare}
\end{table*}

We evaluated our method on linearised fundamental matrix fitting~\cite[Chapter~4]{mvgbook}, where $\bx \in \bbR^8$. We used inlier threshold $\epsilon = 0.03$ for the algebraic residual (convex in $\bx$, hence also quasiconvex), and penalty parameters $\pnt = 1.0$, $\dcrate = 0.5$, $\dcstep = 50$, and $\dclim = 0.01$.

We used three image pairs from VGG~\cite{vgg} (Castle, Valbonne, and Zoom) and three image pairs from sequence \texttt{00} of KITTI odometry~\cite{kitti}~\footnote{CC BY-NC-SA 3.0 License~\cite{kittilicense}.} (frame indices 104-108, 198-201, and 738-742). In each pair, SIFT features~\cite{sift} were extracted and matched using VLFeat~\cite{vlfeat}; Lowe's second nearest neighbour test was also applied to prune matches.

Fig.~\ref{fig:fund_bound} shows the intermediate outputs of Alg.~\ref{main_algo}-F on Castle, KITTI 198-201 and KITTI 738-742, particularly the lower bound of the solution. See Sec.~\ref{supp:sec:real_data} in the supplementary material for the plots for the other image pairs.

Table~\ref{tab:fund_compare} compares our method with the others. Overall, the quality of our method was comparable to the others, with Alg.~\ref{main_algo}-F providing higher quality and tighter bound than Alg.~\ref{main_algo}-E. Note that only our method returned error bounds (Sec.~\ref{sec:error_bound}), which allowed to deduce that Alg.~\ref{main_algo}-F found consensus sets that were close to the optimum. As expected, the fastest methods were the random sampling approaches. Our method was much slower than the others, mainly due to the usage of SA. However, our experiments in Sec.~\ref{sec:synthetic} shows that QA can improve the speed of SA up to a factor of 10 without affecting solution quality (see Fig.~\ref{fig:qa-sa}). Hence, we expect Alg.~\ref{main_algo} to be more competitive as quantum annealer capacity improves. Fig.~\ref{fig:fund_qualitative} qualitatively illustrates the results of Alg.~\ref{main_algo}-E.

\begin{figure}[h]
    \centering

        
        \subfloat[][Zoom]
        {
            \includegraphics[width=0.39\textwidth]{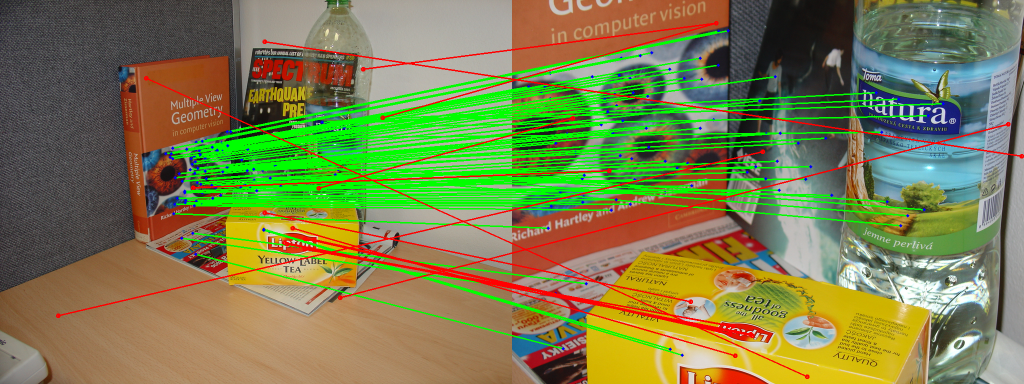}
  
        }
        
        \subfloat[][KITTI 104-108]
        {
            \includegraphics[width=0.39\textwidth]{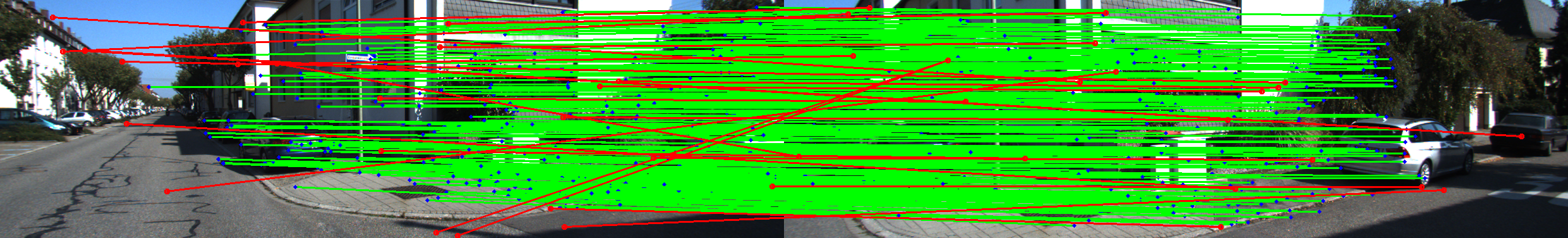}
 
        }
     \vspace{-0.5em}	
     \caption{Qualitative results of Alg.~\ref{main_algo}-E on fundamental matrix estimation. Green and red lines represent inliers and outliers found.}
    \label{fig:fund_qualitative}
\end{figure}

\vspace{-2em}
\subsubsection{Multi-view triangulation} \label{sec:experiment_tri}
\vspace{-0.5em} 

\begin{table*}[h]
    
    	\centering
    	\includegraphics[width=1\textwidth]{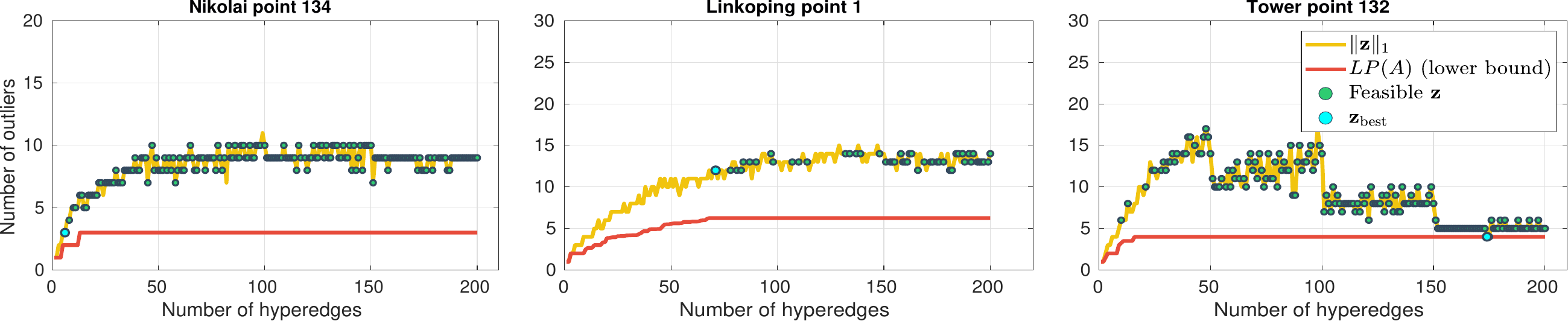}
    	\vspace{-2em}
    	\captionof{figure}{Multi-view triangulation, where number of outliers $\| \bz \|_1$ and lower bound $LP(A)$ plotted across the iterations of Alg.~\ref{main_algo}-F.}
    	\label{fig:tri_bound}
    
     \vspace*{0.5em}  
    
    \centering
    \footnotesize
    \begin{tabular}{|c|c|c|c|c|c|c|c|c|c|}
        \hline
         \multicolumn{2}{|c|}{Method} & RS~\cite{ransac} & LRS~\cite{loransac} & FLRS~\cite{fixingloransac} & EP~\cite{huu2017ep} & IBCO~\cite{zhipeng2018ibco} & QRF~\cite{TJ_ACCV_2020} & Alg.~\ref{main_algo}-E & Alg.~\ref{main_algo}-F \\
         \hline \hline
         Nikolai point 134 & $|\cI|$ (Error bound) & 21 ($-$) & 21 ($-$) & 21 ($-$) & 21 ($-$) & 21 ($-$) & 21 ($-$) & 21 (1.00) & 21 (0.00)\\
         $N = 24$ & Time (s) & 0.24 & 0.32 & 0.30 & 0.34 & 0.36 & 158.39 & 6.12 & 159.28 \\
         \hline
         
         Nikolai point 534 & $|\cI|$ (Error bound) & 16 ($-$) & 16 ($-$) & 16 ($-$) & 15 ($-$) & 17 ($-$) & 17 ($-$) & 16 (2.00) & 16 (1.00) \\
         $N =20$ & Time (s) &  0.27 & 0.35 & 0.25 & 0.29 & 0.32 & 154.63 & 8.71 & 147.14 \\
         \hline
         
         Linkoping point 1 & $|\cI|$ (Error bound) & 15 ($-$) & 15 ($-$) & 15 ($-$) & 14 ($-$) & 16 ($-$) & 14 ($-$) & 13 (5.75) & 13 (5.75)\\
         $N = 25$ & Time (s) & 0.25 & 0.30 & 0.34 & 0.38 & 0.47 & 175.83 & 20.05 & 153.79\\
         \hline
         
         Linkoping point 14 & $|\cI|$ (Error bound) & 36 ($-$) & 36 ($-$) & 36 ($-$) & 35 ($-$) & 37 ($-$) & 32 ($-$) & 37 (4.67) & 37 (4.27) \\
         $N = 52$ & Time (s) & 0.27 & 0.44 & 0.38 & 0.53 & 0.64 & 360.37 & 130.10 & 194.46\\
         \hline
         
         Tower point 3 & $|\cI|$ (Error bound) & 73 ($-$) & 73 ($-$) & 73 ($-$) & 73 ($-$) & 73 ($-$) & 73 ($-$) & 72 (3.00) & 72 (1.00) \\
         $N = 79$ & Time (s) & 0.28 & 0.64 & 0.32 & 0.36 & 0.43 & 555.27 & 27.26 & 177.43 \\
         \hline
         
         Tower point 132 & $|\cI|$ (Error bound) & 79 ($-$) & 79 ($-$) & 79 ($-$) & 79 ($-$) & 81 ($-$) & 81  ($-$) & 79 (2.75) & 81 (0.00) \\
         $N = 85$ & Time (s) & 0.30 & 0.62 & 0.42 & 0.51 & 0.51 & 563.43 & 26.33 & 163.32\\
         \hline
    \end{tabular}
    \vspace{-0.5em}
    \captionof{table}{Multi-view triangulation results.  Alg.~\ref{main_algo} employed simulated annealing (quantum annealing could be faster by $10$ times). Only Alg.~\ref{main_algo}  amongst all methods compared here returned error bounds.}
    \vspace{-1.5em}
    \label{tab:tri_compare}
\end{table*}

Points 134 \& 534 from Nikolai, points 1 \& 14 from Linkoping, and points 3 \& 132 from Tower~\cite{enqvist2011non} were used. In this task, 3D coordinates of those 3D points ($\bx \in \bbR^3$) were estimated using reprojection error (which is quasiconvex~\cite{kahl08}) under outliers. The inlier threshold and penalty were respectively $\epsilon = 1$ pixels and $\pnt=5$. The decay parameters were $\dcrate = 0.5$, $\dcstep = 50$, and $\dclim = 0.03$.

Fig.~\ref{fig:tri_bound} shows the intermediate outputs of Alg.~\ref{main_algo}-F on Nikolai point 134, Linkoping point 1 and Tower point 132, particularly the lower bound of the solution. See Sec.~\ref{supp:sec:real_data} in the supplementary material for the plots for the other points. Interestingly, the results show that it was more difficult to find a tight lower bound here (especially Nikolai point 134 and Linkoping point 1). This could be due to numerical inaccuracies in solving the minimax problem~\eqref{eq:minimax} for quasiconvex residuals~\cite{eppstein,amenta}, which affected the efficacy of hyperedge sampling. Table~\ref{tab:tri_compare} shows the quantitative results; a similar conclusions as that of Table~\ref{tab:fund_compare} can be drawn. In particular, note that only our method was able to provide error bounds; in the case of Tower point 132, the global solution was provably found by the algorithm (gap is zero).

\vspace{-0.5em}
\section{Weaknesses and conclusions}
\vspace{-0.5em} 

There are two main shortcomings: First, Alg.~\ref{main_algo} was validated on an actual quantum computer only for small scale synthetic data (for reasons covered in Sec.~\ref{sec:practical}). To fully realise the potential of the algorithm, testing with real data on a quantum computer is needed. Second, our results reveal that the hyperedge sampling procedure is also crucial to Alg.~\ref{main_algo}. Developing a more effective way of sampling hyperedges is an interesting research direction. 

\vspace{-0.5em}
\paragraph{Conclusions} Our work illustrates the potential of quantum annealing for robust fitting. It outperforms (in simulation) the only other quantum approach in robust fitting~\cite{TJ_ACCV_2020} as well as offers an error bound to mitigate the weakness of current QPU. We hope that our work helps trigger further development on applying quantum computers in robust fitting and computer vision applications.

\vspace{-0.5em}
\section*{Acknowledgement}
This work was supported by Australian Research Council ARC DP200101675, and~D.~Suter acknowledges funding under Australian Research Council grant~DP200103448.

\section*{Supplementary Material}
\begin{alphasection}
\section{Proof of claim~\ref{claim:true_outlier_in_basis}}
\label{supp:sec:proof_true_outlier_in_basis}

By construction, $f(\cI^\ast) = 0$. Given an infeasible basis $\cB$, by monotonicity
\begin{align}
    1 = f(\cB) \le f(\cI^\ast \cup \cB) = 1.
\end{align}
Thus, $\cB$ must contain a point not in $\cI^\ast$.

\section{Proof of claim~\ref{mainclaim}}
\label{supp:sec:proof_main_claim}

Recall that an independent set of a hypergraph is a subset of the vertices where none of the members of the subset form a hyperedge~\cite[Chapter~2]{Bretto2013}.

Given $f(\cI) = 0$, we have that $f(\cA) = 0$ for all $\cA \subseteq \cI$ due to monotonicity. Hence, the vertices in $\cI$ do not form hyperedges, i.e., $\cI$ is an independent set.

Let $\cI$ be an independent set of $H$, and $\Omega = \{\cS_1,\cS_2,\dots\}$ be all $\delta$-subsets of $\cC$. By construction,
\begin{align}
    g(\cS_\ell) \le \epsilon \;\;\; \forall \cS_\ell \in \Omega
\end{align}
since there are no hyperedges in $\cI$. Let $\cS_\ast \in \Omega$ such that
\begin{align}\label{eq:state0}
    g(\cS_\ell) \le g(\cS_\ast) \le \epsilon \;\;\; \forall \cS_\ell \in \Omega.
\end{align}
Suppose $g(\cI) > \epsilon$: from~\cite{eppstein,amenta}, there is a support set $\bar{\cS}$---which is also a basis and hence a $\delta$-subset---of $\cI$ such that
\begin{align}
g(\cI) = g(\bar{\cS}),
\end{align}
which implies that $\bar{\cS} \in \Omega$ but $\bar{\cS} \ne \cS_{\ast}$, and
\begin{align}
    g(\bar{\cS}) > \epsilon \ge g(\cS_\ast),
\end{align}
which contradicts~\eqref{eq:state0}. Thus, we must have that $g(\cI) \le \epsilon$, i.e., $\cI$ is a consensus set.

\section{Reformulating QUBO for D-Wave solvers}
\label{supp:sec:reformulate_qubo_to_dwave}

Recall the QUBO~\eqref{eq:scp_qubo}

\begin{equation*}
    \begin{aligned}
        Q_\lambda(A) = \min_{\bv \in \{0,1\}^{N+\delta^\prime M}}  & \left[ \begin{matrix}\bv^T & 1 \end{matrix}\right] (\bJ + \pnt\bH_A^T\bH_A) \left[ \begin{matrix}\bv^T & 1 \end{matrix}\right]^T
    \end{aligned}
\end{equation*}
Let $\bQ = \bJ + \pnt\bH_A^T\bH_A$, and denote $q_{ij}$ is the element in $i^\text{th}$-row and $j^\text{th}$-column of $\bQ$

D-Wave solvers accept $\bQ$ as a upper-triangular matrix, which can be obtained by following procedure
\begin{itemize}
    \item For every $i$ and $j$, if $j > i$, then $q_{ij} = q_{ij} + q_{ji}$
    \item For every $i$ and $j$, if $j < i$, then $q_{ij} = 0$
\end{itemize}
We attain the new QUBO with upper-triangular $\bQ$
\begin{equation*}
    \begin{aligned}
        Q_\lambda(A) = \min_{\bv \in \{0,1\}^{N+\delta^\prime M}}  & \left[ \begin{matrix}\bv^T & 1 \end{matrix}\right] \bQ \left[ \begin{matrix}\bv^T & 1 \end{matrix}\right]^T
    \end{aligned}
\end{equation*}
which, however, still cannot be applied to D-Wave solvers since the variables are in the form $\left[ \begin{matrix}\bv^T & 1 \end{matrix}\right]$. However, this formulation can be rewritten using a simple derivation.  

Suppose 
\begin{align}
\bv &= \left[ \begin{matrix} v_1 & v_2 & v_3\end{matrix} \right] \\
\bQ & = \left[ \begin{matrix}q_{11} & q_{12} & q_{13} & q_{14} \\ 0 & q_{22} & q_{23} & q_{24} \\ 0 & 0 & q_{33} & q_{34} \\ 0 & 0 & 0 & q_{44} \end{matrix} \right].
\end{align}

We then take the last column of $\bQ$ except the last element $q_{44}$, which yields $\bq = \left[\begin{matrix} q_{14} & q_{24} & q_{34} \end{matrix} \right]^T$. Next, we add $\bq$ to the diagonal of $\bQ$
\begin{align}
    \bQ^\prime = \left[ \begin{matrix} q_{11} + q_{14} & q_{12} & q_{13} \\ 0 & q_{22} + q_{24} & q_{23} \\ 0 & 0 & q_{33} + q_{34} \end{matrix} \right].
\end{align}
Since $\bv$ is a binary vector, i.e., $v_i^2 = v_i$, we can obtain

\begin{align}
    \left[ \begin{matrix}\bv^T & 1 \end{matrix}\right] \bQ \left[ \begin{matrix}\bv^T & 1 \end{matrix}\right]^T = \bv^T \bQ^\prime \bv + q_{44}.
\end{align}
Therefore, we get the new QUBO
\begin{equation}
    \begin{aligned}
        Q_\lambda(A) = \min_{\bv \in \{0,1\}^{N+\delta^\prime M}}  & \bv^T \bQ^\prime \bv + \text{constant},
    \end{aligned}
    \label{supp:eq:final_qubo}
\end{equation}
which can be directly applied to D-Wave solvers.



\section{Spectral gap} \label{supp:sec:spectral_gap}

\paragraph{Computation of spectral gap.} Recall matrix $\bQ^\prime$ of QUBO~\eqref{supp:eq:final_qubo}. For simplicity, let 
\begin{align}
    n = N + \delta^\prime M
\end{align}
thus $\bQ^\prime$ is an \textit{upper-triangular matrix} of the size $n \times n$, and denote $q^\prime_{i,j}$ is the element in $i^\text{th}$-row and $j^\text{th}$-column of $\bQ^\prime$.

The QUBO problem~\eqref{supp:eq:final_qubo} is firstly converted to Ising problem~\cite{djidjev2018efficient}
\begin{align}
    h_i = & \frac{q^\prime_{ii}}{2} + \sum_{j=1}^n \frac{q^\prime_{ij}}{4}, \\
    J_{ij} = & \frac{q^\prime_{ij}}{4}
\end{align}
for all $i \in \{1, \dots, n \}$ and all $i < j$. In QPU, $h_i$ are termed biases, and $J_{ij}$ are called couplings. Then, biases and couplings are normalised such that $h_i \in [-2,2]$ and $J_{ij} \in [-1,1]$, since D-Wave limits the value range of biases in $[-2,2]$ and couplings in $[-1,1]$~\cite{dwavesolverproperties}. Next, the initial and final Hamiltonians are computed
\begin{align}
    H_\text{init} = & \sum_i \hat{\sigma}_x^{(i)}, \\
    H_\text{final} = & \sum_i h_i \hat{\sigma}_z^{(i)} + \sum_{i<j} J_{ij} \hat{\sigma}_z^{(i)}\hat{\sigma}_z^{(j)} 
\end{align}
where,
\begin{align*}
    \hat{\sigma}_x^{(i)} = &\overbrace{\bI \otimes \bI \otimes \dots \otimes \bI \otimes \tikzmark{a}\sigma_x \otimes \bI \otimes \dots \otimes \bI}^N, \\
    \\
    \hat{\sigma}_z^{(i)} = &\overbrace{\bI \otimes \bI \otimes \dots \otimes \bI \otimes \tikzmark{b}\sigma_z \otimes \bI \otimes \dots \otimes \bI}^N, \\
    \\
    \hat{\sigma}_z^{(i)}\hat{\sigma}_z^{(j)} = & \overbrace{\bI \otimes \dots \otimes \tikzmark{c}\sigma_z \otimes \bI \otimes \dots \otimes \bI \otimes \tikzmark{d}\sigma_z \otimes \dots \otimes \bI}^N, \\
    \\
    \bI = & \left[\begin{matrix}1 & 0 \\ 0 & 1 \end{matrix} \right], \quad \sigma_x = \left[\begin{matrix} 0 & 1 \\ 1 & 0 \end{matrix} \right], \quad \sigma_z = \left[\begin{matrix} 1 & 0 \\ 0 & -1 \end{matrix} \right],
\end{align*}
the Hamiltonian of the quantum computer is represented as
\begin{align}
    H(s) = (1-s)H_\text{init} + sH_\text{final},
\end{align}
where $s \in [0,1]$ is the normalised annealing time. 
\begin{tikzpicture}[remember picture,overlay]
\draw[<-] 
      ([shift={(2pt,-2pt)}]a) |- ([shift={(-10pt,-10pt)}]a) 
      node[anchor=east] {$\scriptstyle i^\text{th}\text{ position}$}; 
\draw[<-] 
      ([shift={(2pt,-2pt)}]b) |- ([shift={(-10pt,-10pt)}]b) 
      node[anchor=east] {$\scriptstyle i^\text{th}\text{ position}$}; 
\draw[<-] 
      ([shift={(2pt,-2pt)}]c) |- ([shift={(-10pt,-10pt)}]c) 
      node[anchor=east] {$\scriptstyle i^\text{th}\text{ position}$}; 
\draw[<-] 
      ([shift={(2pt,-2pt)}]d) |- ([shift={(-10pt,-10pt)}]d) 
      node[anchor=east] {$\scriptstyle j^\text{th}\text{ position}$};
\end{tikzpicture}

For a particular $s$, $H(s)$ is a $2^N\times 2^N$ matrix, which is then decomposed to obtain the smallest and second smallest eigenvalues. The eigenspectra of smallest and second smallest eigenvalues respectively represent the ground state and the first excited (high) energy state. The minimum gap between those two eigenspectra represents the spectral gap.

\paragraph{Results.} The synthetic data with the setting same as that of Sec.~\ref{sec:synthetic} is generated, where $N = 5$, outlier ratio $= 0.4$, $\pnt \in [0.1, 100]$, and $A$ containing all hyperedges $E$. 

\begin{figure}
    	\centering
		\includegraphics[width=0.4\textwidth]{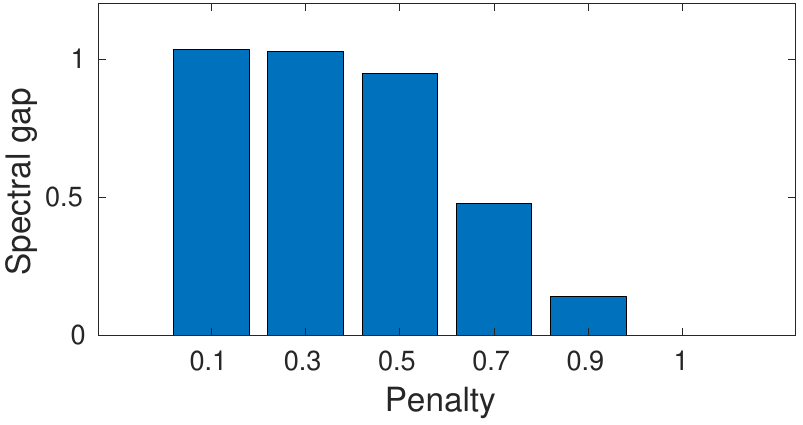}
    	\captionof{figure}{Spectral gap.}
    	\label{supp:fig:spectral_gap}
\end{figure}

As the penalty $\pnt$ increases, the spectral gap quickly reduces (Fig.~\ref{supp:fig:spectral_gap}). This implies that the probability of the quantum system remains in the ground state during the annealing time decreases with larger $\pnt$.

The eigenvalues of specific penalty $\pnt$ is also shown in Fig.~\ref{supp:fig:eigens}, which indicates the decrease of spectral gap with larger penalty $\pnt$.

\begin{figure}
        \centering
        
            \subfloat[][]
            {
                \includegraphics[width=0.45\textwidth]{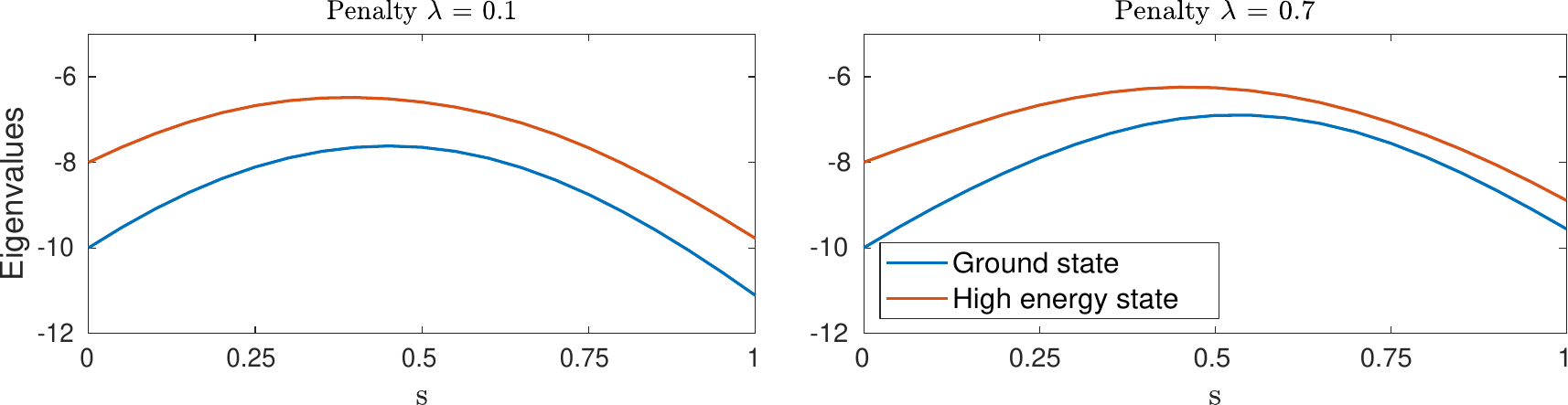}.
                \label{supp:fig:small_eigens}
            }
            
            \subfloat[][]
            {
                \includegraphics[width=0.45\textwidth]{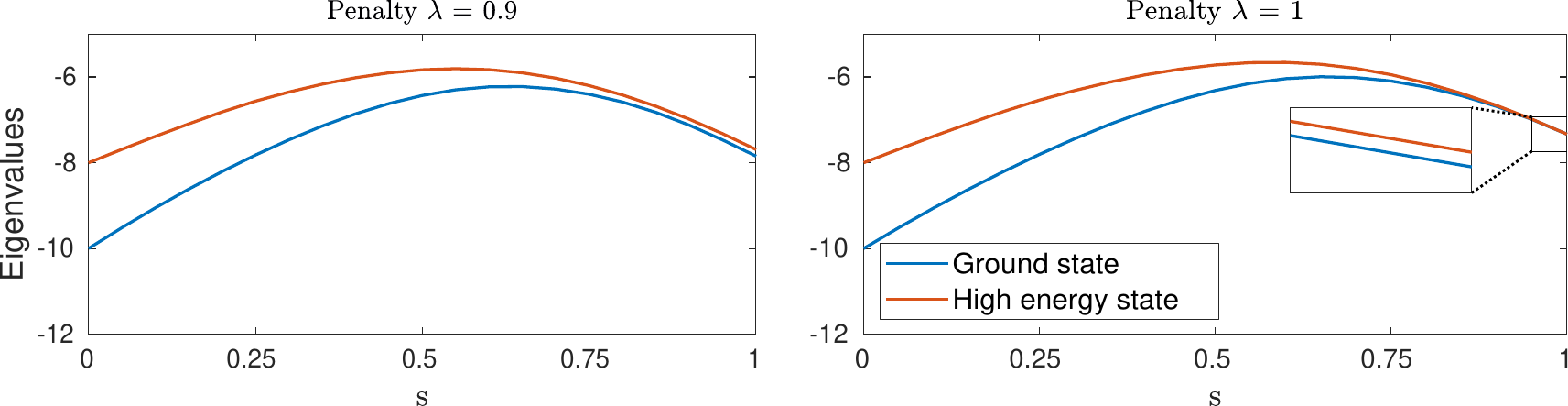}.
                \label{supp:fig:big_eigens}
            }
        \caption{Eigenvalues of ground state and first excited (high) energy state for every penalty value $\lambda$, where the minimum gap between two eigenspectrums is the spectral gap.}
        \label{supp:fig:eigens}
        \vspace{-1em}
\end{figure}

\section{Benefit of penalty decay} \label{supp:sec:penalty_decay}
Fig.~\ref{supp:fig:pnt_decay} shows the comparison between fixing $\pnt$ and decaying $\pnt$ in fundamental matrix estimation. The decay parameters are set same as those in Sec.~\ref{sec:experiment_fund}. If the penalty $\pnt$ is large ($\pnt = 1$), Alg.~\ref{main_algo} can quickly find a consensus set but the error bounds cannot be tightened. By contrast, if the penalty $\pnt$ is small ($\pnt = 0.02$), Alg.~\ref{main_algo} will require more iterations to find a consensus set. Therefore, decaying penalty $\pnt$ is a reasonable strategy that helps Alg.~\ref{main_algo} quickly find a consensus set and efficiently tighten the error bound.
\begin{figure*}[h]
        \centering
       
        \subfloat[][Castle]
            {
                \includegraphics[width=0.9\textwidth]{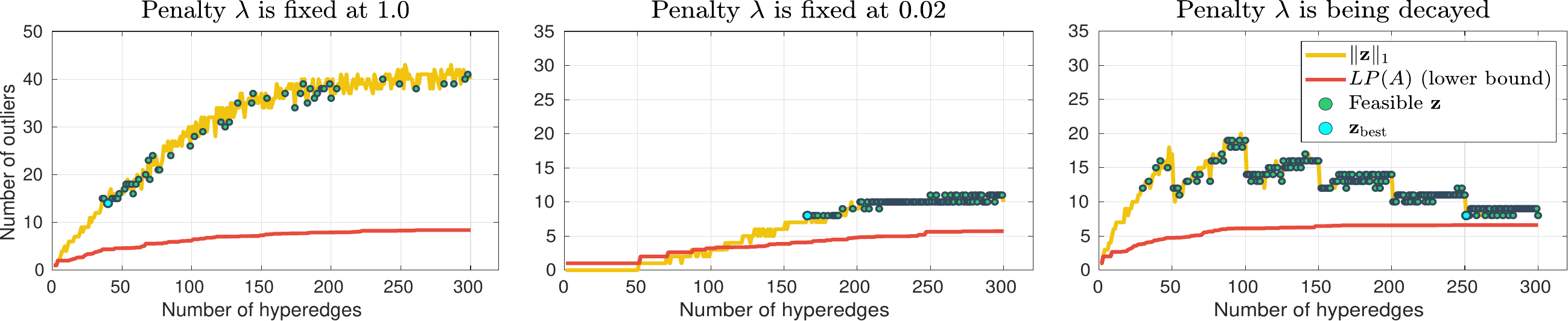}.
             
            }
            
            \subfloat[][KITTI 198-201]
            {
                \includegraphics[width=0.9\textwidth]{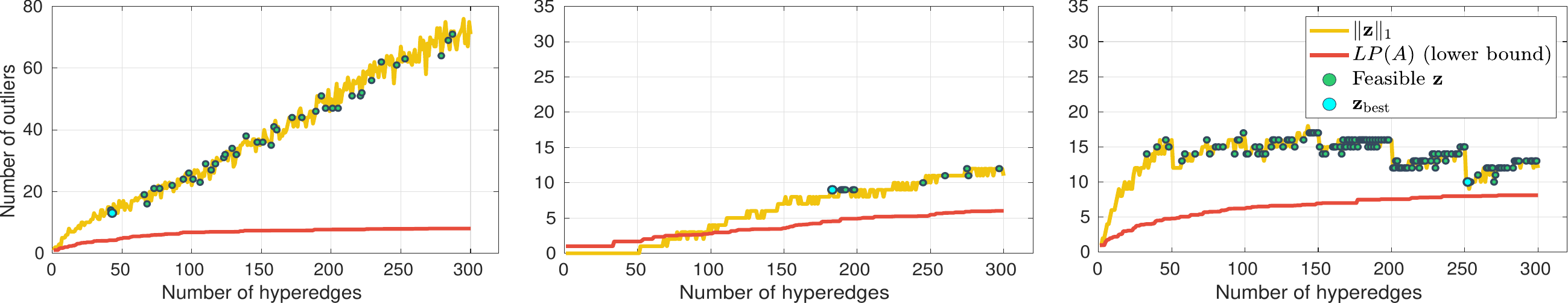}.
               
            }
            
            \subfloat[][KITTI 738-742]
            {
                \includegraphics[width=0.9\textwidth]{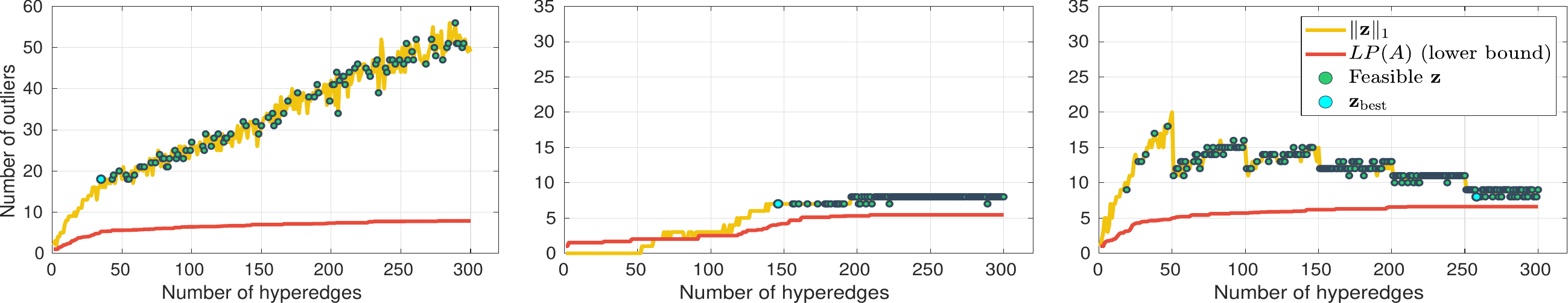}
               
            }

        \caption{Comparing fixed penalty $\pnt$ to penalty $\pnt$ being decayed. The error bound cannot be tightened with a large penalty (left), while Alg.~\ref{main_algo} requires more iterations to find a consensus set (middle). Therefore, decaying penalty can quickly find a consensus set and efficiently tighten the error bound (right).}
        \label{supp:fig:pnt_decay}
\end{figure*}

\section{Minor embedding} \label{supp:sec:embedding}

Before quantum annealing, minor embedding should be performed to embed $\bQ^\prime$ (Eq.~\eqref{supp:eq:final_qubo}) to the QPU topology.

\begin{figure}
    \centering
    \subfloat[][$N=20$]
    {
        \includegraphics[width=0.45\textwidth]{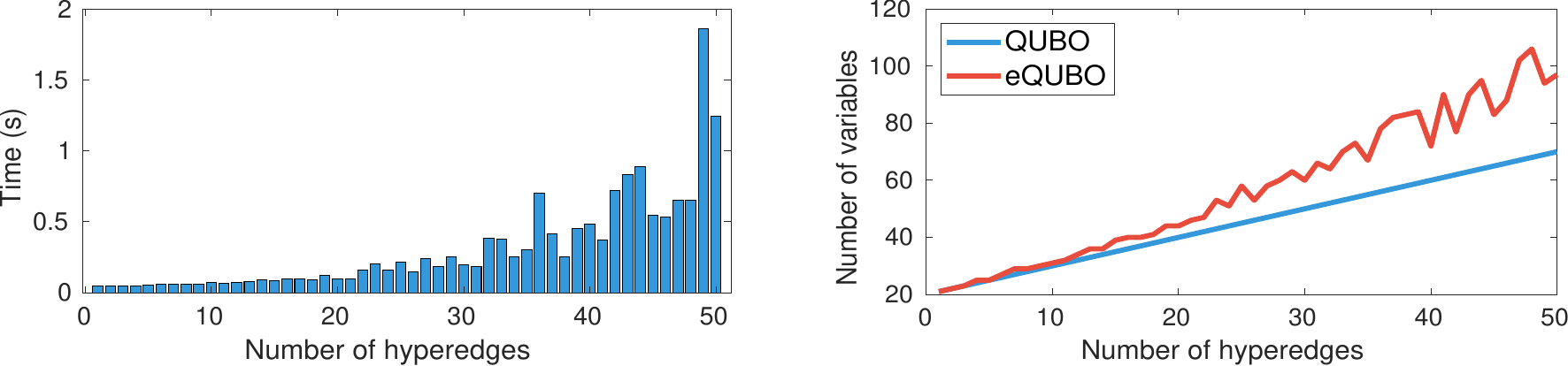}
    }
    
    \subfloat[][$N=50$]
    {
        \includegraphics[width=0.45\textwidth]{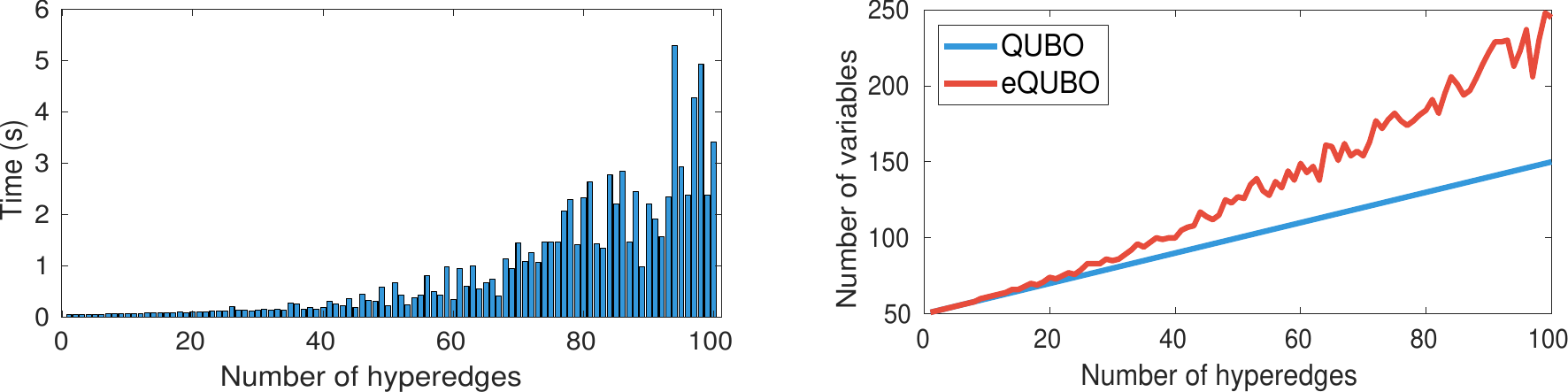}
    }
    
    \subfloat[][$N=100$]
    {
        \includegraphics[width=0.45\textwidth]{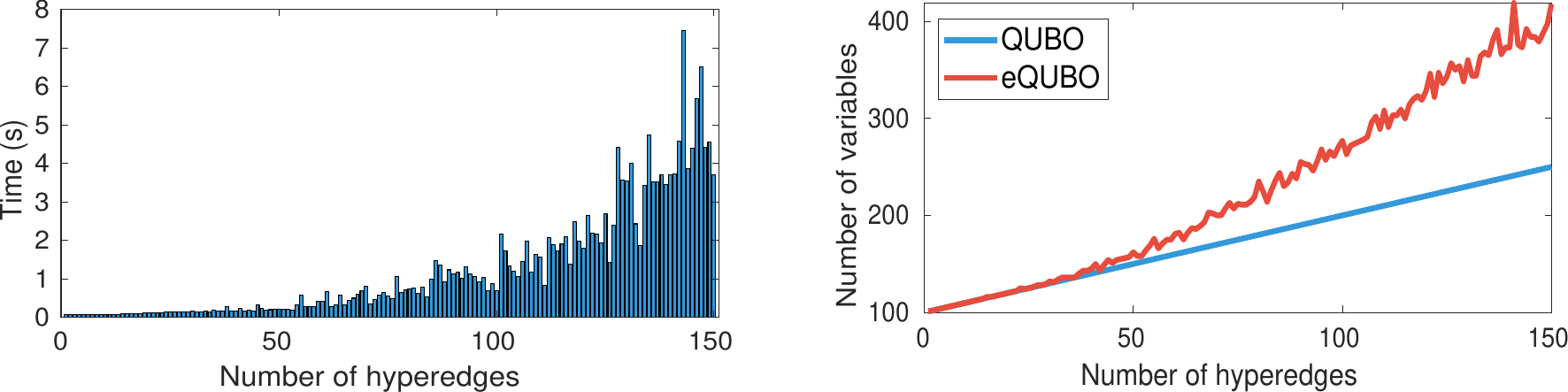}
    }
    
    \caption{The embedding time (left) and total number of variables (right) of QUBO~\eqref{supp:eq:final_qubo} before and after embedding to QPU topology (respectively denoted as QUBO and eQUBO).}
    \vspace{-1em}
    \label{supp:fig:embedding}
\end{figure}

To investigate the minor embedding, synthetic data with the setting same as that of Sec.~\ref{sec:synthetic} is generated, where $N = 20, 50, 100$ and outlier ratio $=0.2$. In every iteration of Alg.~\ref{main_algo}, we measure the embedding time and total number of variables of QUBO~\eqref{supp:eq:final_qubo} before and after embedding  (see Fig.~\ref{supp:fig:embedding}). In all cases, the embedding time and number of variables increase as more hyperedges are sampled. Also see Fig.~\ref{supp:fig:embedding_vis} for the visualisation of the embedding on the QPU.

\begin{figure*}
    \centering
    \subfloat[][(From left to right) $N$=5 (\#~hyperedges = 4), $N$=10 (\#~hyperedges = 16), $N$ = 20 (\#~hyperedges = 50), and $N$ = 30 (\#~hyperedges = 80).]
    {
        \includegraphics[width=1\textwidth]{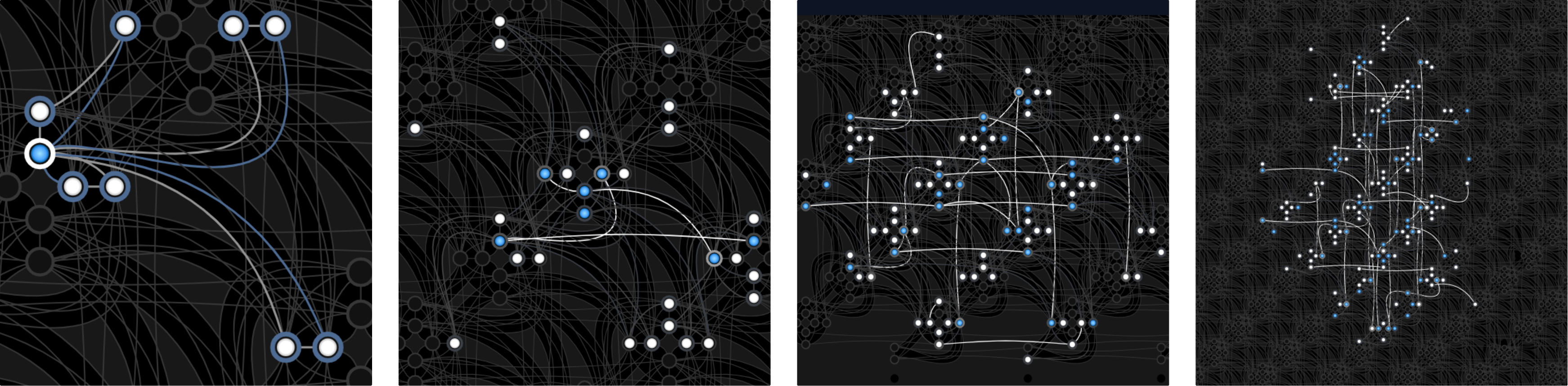}
    }
    
    \subfloat[][(From left to right) $N$=50 (\#~hyperedges = 100), and $N$=100 (\#~hyperedges = 150)]
    {
        \includegraphics[width=1\textwidth]{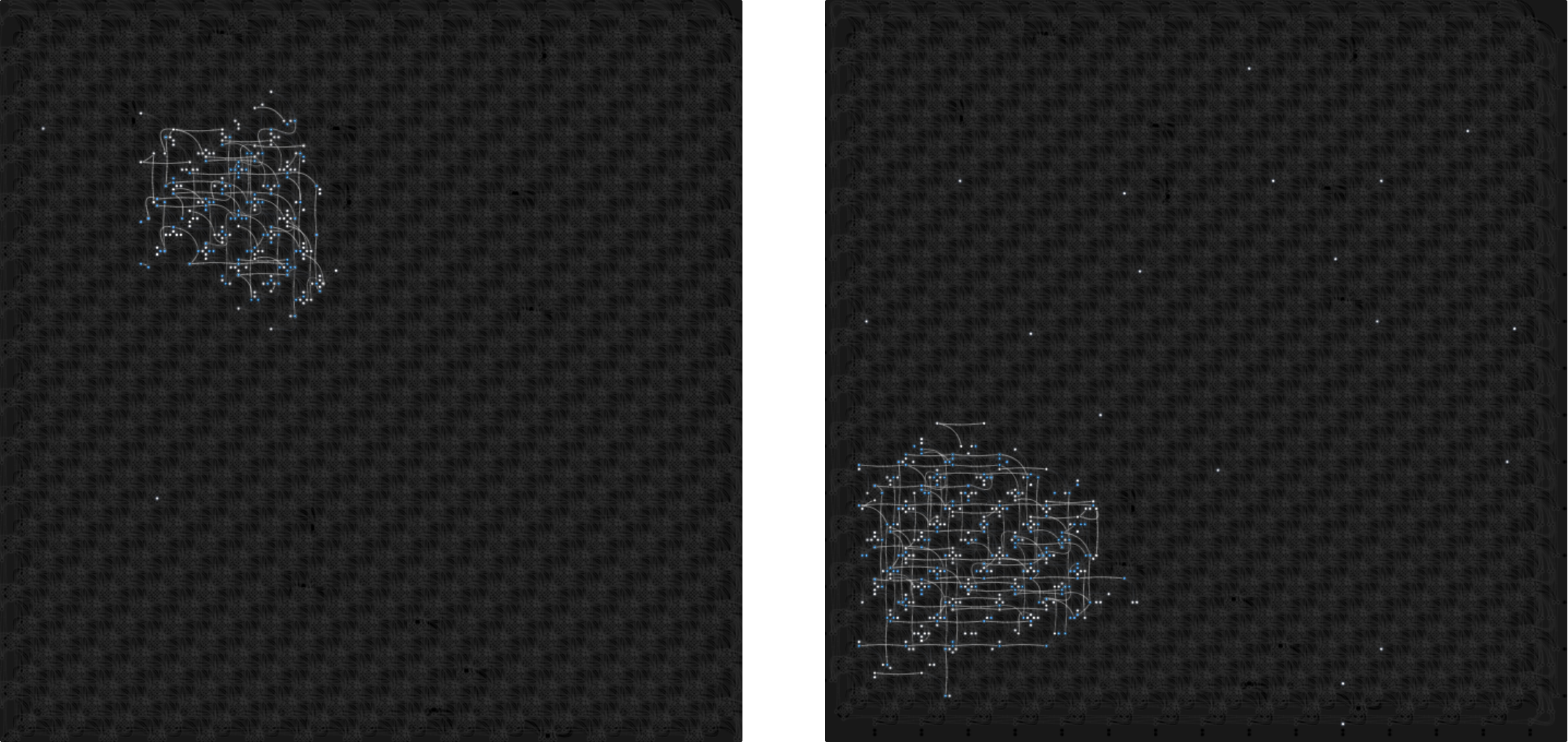}
    }

    \caption{Visualisation of the embedding using D-Wave problem inspector~\cite{dwaveinspector}. Each node represents a physical qubit. In each node, the colors in the outer rings represents the signs of qubit biases measured in the lowest energy state; and the inner colors represent the solutions. Edges represent the coupling strengths.  }
    \label{supp:fig:embedding_vis}
\end{figure*}

\section{More experimental results on real data} \label{supp:sec:real_data}
\paragraph{Fundamental matrix estimation.} Fig.~\ref{supp:fig:fund_bound} shows the intermediate outputs of Alg.~\ref{main_algo}-F on Zoom, Valbonne and KITTI 104-108. A same conclusion as Sec.~\ref{sec:experiment_fund} can be drawn.
\begin{figure*}[h]
    	\centering
		\includegraphics[width=1\textwidth]{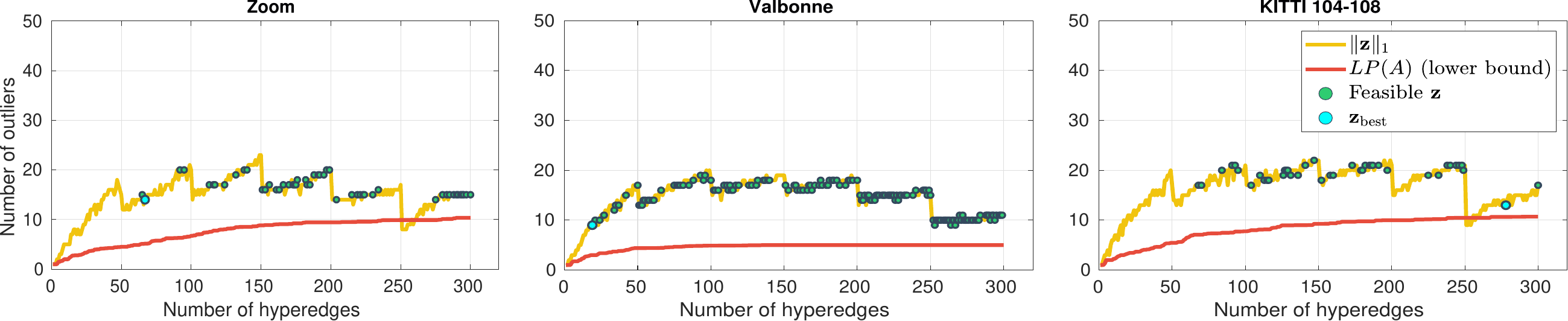}
    	\captionof{figure}{Fundamental matrix estimation, where number of outliers $\| \bz \|_1$ and lower bound $LP(A)$, plotted across the iterations of Alg.~\ref{main_algo}-F.}
    	\label{supp:fig:fund_bound}
\end{figure*}

\paragraph{Triangulation.} Fig.~\ref{supp:fig:triangulation_bound} shows the Alg.~\ref{main_algo}-F's intermediate outputs on Nikolai point 534, Linkoping point 14 and Tower point 3. A same conclusion as Sec.~\ref{sec:experiment_tri} can be drawn.
\begin{figure*}[h]
    	\centering
		\includegraphics[width=1\textwidth]{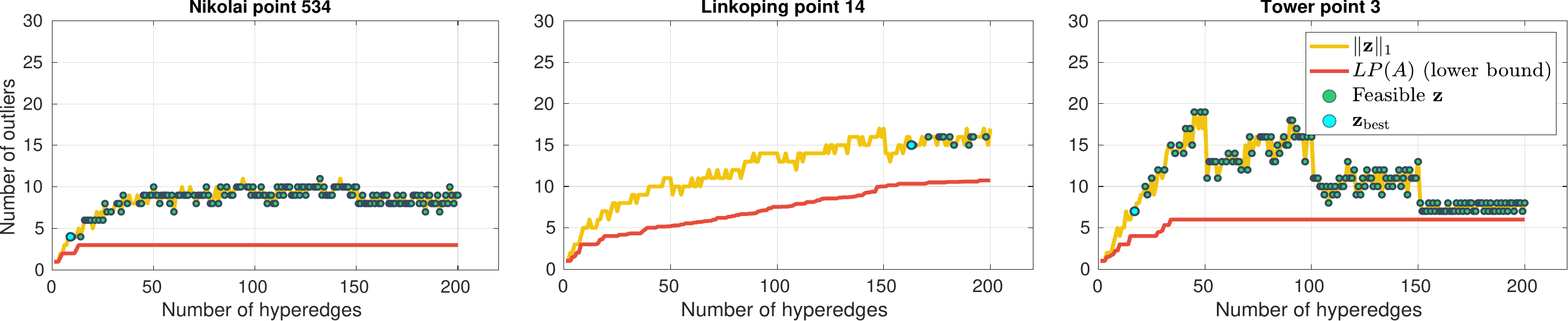}
    	\captionof{figure}{Multi-view triangulation, where number of outliers $\| \bz \|_1$ and lower bound $LP(A)$ plotted across the iterations of Alg.~\ref{main_algo}-F.}
    	\label{supp:fig:triangulation_bound}
\end{figure*}

\end{alphasection}
{\small
\bibliographystyle{ieee_fullname}
\bibliography{bib}
}

\end{document}